\documentclass[11pt,a4paper]{article}

\usepackage[utf8]{inputenc}
\usepackage[T1]{fontenc}
\usepackage{lmodern}
\usepackage[round]{natbib}
\usepackage[english]{babel}
\usepackage{microtype}
\usepackage[a4paper,margin=2.4cm,headheight=14.5pt]{geometry}
\usepackage{amsmath,amssymb,mathtools}
\usepackage{amsthm}
\numberwithin{equation}{section}

\usepackage[dvipsnames]{xcolor}
\definecolor{axinfo}{RGB}{21,96,158}
\definecolor{axenergie}{RGB}{182,63,30}
\definecolor{soudure}{RGB}{112,44,140}
\definecolor{okD}{RGB}{0,118,58}
\definecolor{okDH}{RGB}{0,112,118}
\definecolor{okC}{RGB}{196,106,0}
\definecolor{okS}{RGB}{168,32,52}

\usepackage{tikz}
\usetikzlibrary{arrows.meta,decorations.pathreplacing,patterns,babel,%
                positioning,calc,fit,backgrounds}

\graphicspath{{companion/figures/}{../companion/figures/}}
\usepackage{booktabs}
\usepackage{caption}
\usepackage{enumitem}
\setlist{itemsep=2pt,topsep=4pt,parsep=0pt}

\usepackage{fancyhdr}
\usepackage[colorlinks=true,linkcolor=axinfo!85!black,
            citecolor=okD!80!black,urlcolor=soudure]{hyperref}
\usepackage[capitalise,noabbrev]{cleveref}

\theoremstyle{plain}
\newtheorem{theoreme}{Theorem}[section]
\newtheorem{proposition}[theoreme]{Proposition}
\newtheorem{lemme}[theoreme]{Lemma}
\newtheorem{corollaire}[theoreme]{Corollary}
\newtheorem{principe}{Principle}
\theoremstyle{definition}
\newtheorem{definition}[theoreme]{Definition}
\newtheorem{premisse}{Assumption}
\newtheorem{postulat}{Postulate}

\theoremstyle{remark}
\newtheorem{remarque}[theoreme]{Remark}
\theoremstyle{definition}
\newtheorem{problem}{Problem}

\crefname{theoreme}{theorem}{theorems}\Crefname{theoreme}{Theorem}{Theorems}
\crefname{proposition}{proposition}{propositions}\Crefname{proposition}{Proposition}{Propositions}
\crefname{lemme}{lemma}{lemmas}\Crefname{lemme}{Lemma}{Lemmas}
\crefname{corollaire}{corollary}{corollaries}\Crefname{corollaire}{Corollary}{Corollaries}
\crefname{principe}{principle}{principles}\Crefname{principe}{Principle}{Principles}
\crefname{definition}{definition}{definitions}\Crefname{definition}{Definition}{Definitions}
\crefname{premisse}{assumption}{assumptions}\Crefname{premisse}{Assumption}{Assumptions}
\crefname{postulat}{Postulate}{Postulates}\Crefname{postulat}{Postulate}{Postulates}
\crefname{hypothese}{hypothesis}{hypotheses}\Crefname{hypothese}{Hypothesis}{Hypotheses}
\crefname{remarque}{remark}{remarks}\Crefname{remarque}{Remark}{Remarks}
\crefname{problem}{problem}{problems}\Crefname{problem}{Problem}{Problems}
\crefname{equation}{equation}{equations}\Crefname{equation}{Equation}{Equations}
\crefname{figure}{figure}{figures}\Crefname{figure}{Figure}{Figures}
\crefname{table}{table}{tables}\Crefname{table}{Table}{Tables}

\usepackage[framemethod=TikZ]{mdframed}
\theoremstyle{definition}
\newmdtheoremenv[linecolor=okDH!65!black,linewidth=0.8pt,backgroundcolor=okDH!5,%
  roundcorner=3pt,innertopmargin=7pt,innerbottommargin=8pt,%
  innerleftmargin=9pt,innerrightmargin=9pt,skipabove=9pt,skipbelow=9pt]{encadre}{Box}
\crefname{encadre}{box}{boxes}\Crefname{encadre}{Box}{Boxes}

\newcommand{\tierD}{\textcolor{okD!75!black}{\textbf{\small(D)}}}
\newcommand{\tierE}{\textcolor{soudure!80!black}{\textbf{\small(E)}}}
\newcommand{\tierP}{\textcolor{axenergie!82!black}{\textbf{\small(P)}}}

\newcommand{\kB}{k_{\mathrm B}}
\newcommand{\dd}{\mathrm{d}}
\newcommand{\Rbb}{\mathbb{R}}
\newcommand{\Wtwo}{\mathcal{W}_2}
\newcommand{\Ptwo}{\mathcal{P}_2}
\newcommand{\Acal}{\mathcal{A}}
\newcommand{\Ecal}{\mathcal{E}}
\newcommand{\Gcal}{\mathcal{G}}
\newcommand{\Ucal}{\mathcal{U}}
\newcommand{\Ccal}{\mathcal{C}}
\newcommand{\Fcal}{\mathcal{F}}
\newcommand{\Wcal}{\mathcal{W}}
\newcommand{\md}[1]{|\dot{#1}|}
\newcommand{\Ent}{\operatorname{Ent}}
\newcommand{\Var}{\operatorname{Var}}

\hypersetup{pdftitle={A Transport-Based Geometry \\ of Belief-Cost},
            pdfauthor={Laurent Caraffa}}

\begin{document}

% =============================================================================
%  (a) BLOC TITRE + SOUS-TITRE
% =============================================================================
\begin{center}
  {\LARGE\bfseries A Transport-Based Geometry\\[2pt] of Belief-Cost}\\[9pt]
  {\large Laurent Caraffa}\\[3pt]
  Université Gustave Eiffel, LASTIG, IGN-ENSG\\[2pt]
  \texttt{laurent.caraffa@ign.fr}\\[6pt]
  \today\\[3pt]
  {\small\itshape\color{black!55}preliminary version \textemdash{} comments welcome}
\end{center}

% =============================================================================
\bigskip
\noindent\textbf{Abstract.}
A finite agent, a machine's digital twin or any bounded reasoner, infers a fixed and noisy world
through finite sensors, so its coherent output is a \emph{belief}: a probability density over states
(the Bayes posterior). Such an agent stops short of certainty, and revising a belief carries a cost. We
propose a framework for belief costs based on optimal transport, motivated by these facts. We
pose two postulates. \textbf{P0} \emph{(the arena)}: a revision cost is a scalar price on optimal
transport, so beliefs live in Wasserstein space. \textbf{P1} \emph{(uniform pricing)}: one nat of
knowledge costs the same metric length everywhere, the \emph{eikonal} condition. Among conceivable
pricing rules we study this one. Under P0 and P1 the cost metric is optimal transport conformally
reweighted by Fisher information, $\tilde g_{e,U}=2(e+U)\,g_{\Wtwo}$, and the Fisher family is a
\emph{characterization}: among continuous reliefs, uniform pricing is equivalent to $U=cJ$. Two
consequences follow on the conformal class. Certainty sits at infinite cost-distance once the relief
dominates the Fisher information, so a well-posed inference has a cost floor diverging at certainty
(necessity conjectural beyond power laws). On location-scale leaves the geometry is hyperbolic, and the Stam bound places the
Gaussian as the most curved one (at $e=0$). The results are geometric, in nats, and hold up to
units: a change of cost unit rescales all distances and preserves every conclusion (boundary,
eikonal family, hyperbolicity, Gaussian extremum), a gauge theorem; a global change of state
units at $e=0$ is an isometry; the content lies in signs, rankings and ratios. Via Landauer (one nat worth $\kB T$) the cost floor becomes an
energy floor: revising toward certainty would demand unbounded energy. Physics anchors the unit and
enters no theorem. Removing either
postulate leaves the selection open.

\section{Introduction}
\label{sec:intro}\label{nv4:sec:intro}

A \emph{digital twin} \citep{grieves2017} maintains, inside a machine, an image of a real system (a territory, a city, a structure) in order to track or anticipate it. It is often
conceived as a \emph{replica}: a deterministic copy, regularly refreshed with data.
In the setting we consider, however, the machine that realizes it is \emph{finite} and
accesses the system only through finite, noisy channels, whether a physical sensor or a
\emph{finite agent} (a human included), each itself noisy. Its coherent output is then a \emph{belief}: a probability density over the
possible states, produced by Bayesian inference, the coherent framework for reasoning under
uncertainty \citep{cox1946,jaynes2003,mackay2003}. Practice already does this: ensembles,
data assimilation, Kalman filters \citep{kalman1960}, often implicitly. \Cref{fig:probleme} shows this
acquisition pipeline: a fixed hidden state, a finite noisy sensor, and the Bayes posterior it yields.

\begin{figure}[t]
\centering
\includegraphics[width=\linewidth]{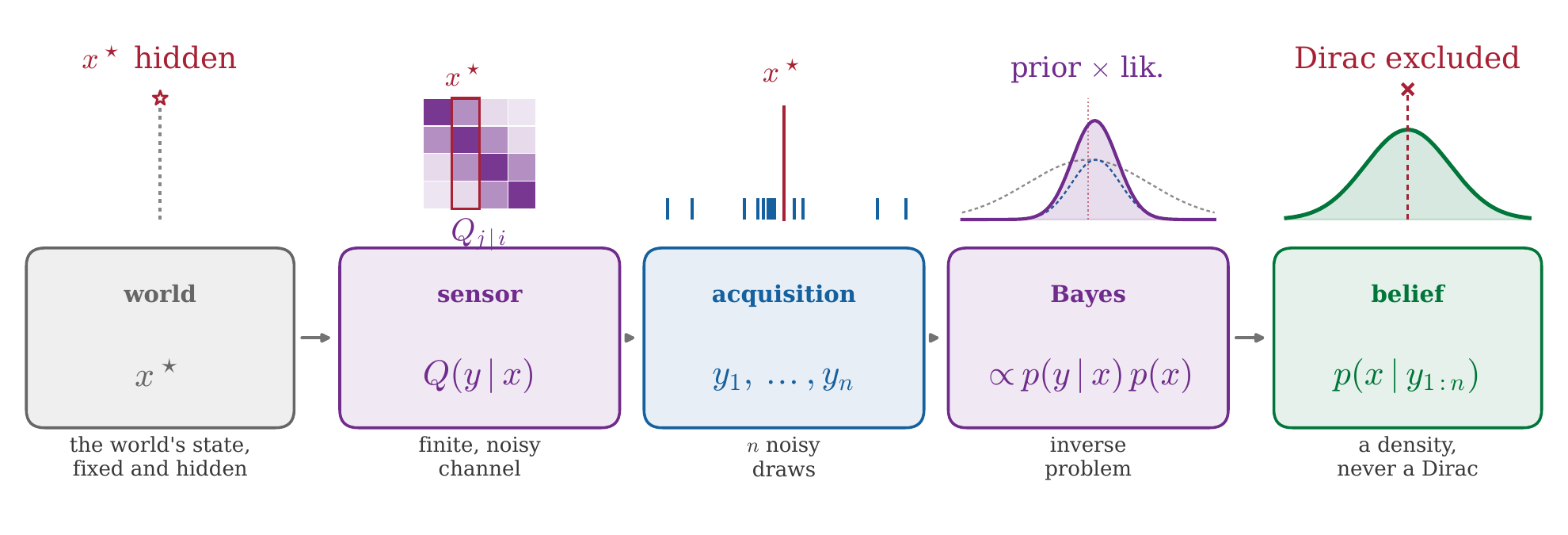}
\caption{\textbf{The inference chain: from a noisy world to a belief}. The state
$x^\star$ (fixed, hidden); the finite and noisy sensor $Q(y\mid x)$
($Q_{j\mid i}=P(y{=}b_j\mid x{=}a_i)$); the $n$ draws $y_1,\dots,y_n$, scattered around
$x^\star$ by the noise; the belief $p(x\mid y_{1:n})\propto p(y\mid x)\,p(x)$ obtained
by Bayes. The output is a \emph{density} (of entropy $H=-\int p\ln p$); certainty,
the barred centered Dirac, is excluded from it.}
\label{fig:probleme}
\end{figure}

To such a belief, \emph{certainty} is out of reach, for two independent reasons. First, reaching it
would demand unbounded information: the
precision of any estimate is limited by the information behind it (\citealp{fisher1925}; the
Cramér--Rao bound, \citealp[Th.~11.10.1]{coverthomas2006}), and a finite agent has only a finite
amount. Second, it would demand unbounded energy: erasing the last of the uncertainty dissipates
heat (the Landauer principle, \citealp{landauer1961}). The two finitudes converge: certainty stays
out of reach as the belief sharpens toward it. This limit also keeps belief revisable. A belief held with full
certainty could never be revised, since a probability of $0$ or $1$ admits no update (Cromwell's
rule, \citealp{lindley1991}). By holding the agent
short of it, finiteness keeps every belief open to revision: any two beliefs stay revisable into one another.

To such an agent, changing its mind costs: it pays, down to a single erased bit. We turn this
finiteness into geometry (the geometry of what it costs to change one's mind): its distances are
revision costs, its \emph{boundary} is certainty (the excluded Dirac), its curvature the object to
characterize. A geometric toolbox already exists, and each piece, in its own way, meets that
boundary: information geometry equips beliefs with the Fisher--Rao metric \citep{chentsov1972,amari2016},
under which the Dirac lies at \emph{infinite} distance, a measure of
\emph{distinguishability} rather than of cost; optimal transport metrizes the displacement of mass
\citep{otto2001}, where the Dirac stays at \emph{finite} distance; the thermodynamics
of information prices the erasure of certainty (Landauer) and the dissipation of a driven belief
\citep{sivak2012,aurell2011,ito2023}, on a \emph{given} geometry; and Bayesian mechanics ties a
dissipative agent to a geometry of belief on a Fisher (distinguishability) manifold \citep{friston2019,sakthivadivel2022}.
Each of these supplies a piece; we characterize a geometry of belief-\emph{cost} in which certainty
is an \emph{uncrossable boundary} at infinite cost-distance, carrying a curvature. Where Fisher--Rao is the canonical geometry of \emph{distinguishability} (\v{C}encov), the analogue for \emph{revision cost} is a Fisher-reweighted transport metric: theorem-grade on the power family, conjectural in general (Conjecture~C1, \cref{nv4:rem:C1}). \v{C}encov's metric follows from an invariance; this one follows from posed choices, made precise in \cref{sec:cadre}, and carries a \emph{metrological} invariance of its own: every conclusion is stated in nats and holds up to the units of cost (\cref{nv4:thm:jauge}) and of the state (\cref{nv4:prop:etat-unites}). The same Fisher information serves \v{C}encov as a component of the distinguishability ruler and serves here as a price: one scalar, two roles. This organizes the paper (\cref{nv4:rem:amari}).

We propose an \emph{axiomatic framework} for transport-based belief costs; finiteness
motivates the postulates. We pose two, an arena (\textbf{P0}) and a uniform pricing of information
(\textbf{P1}), delimit them, and characterize their consequences. Other arenas are conceivable (an
information-geometric Fisher--Rao metric, a parameter-space one) and other prices are conceivable
(proportional to energy, or to a divergence); we note them where they arise and study one
transport-based framework. The central result is a \emph{characterization}: among continuous reliefs,
uniform pricing is equivalent to the Fisher family $U=cJ$, so Fisher information follows from the two
postulates. The framework runs on four levels, by status: the agent's finiteness (motivation),
\textbf{P0} the arena (modelling choice), \textbf{P1} uniform pricing (calibration choice), and the
three results (mathematical consequences).

\smallskip\noindent\textbf{Plan.} \cref{sec:apercu} sketches the proposed formulation in broad
strokes; \cref{sec:cadre} makes it formal (the arena, the two postulates, the cost class);
\cref{nv4:sec:carac} proves the characterizations: the boundary at infinite cost-distance, the eikonal $U=cJ$, and
hyperbolicity with Stam rigidity; \cref{nv4:sec:consequences} draws from them a well-posed inference
with a cost floor; \cref{nv4:sec:ouvert} delimits scope and open problems and \cref{sec:travaux}
situates the work; the machinery proofs are deferred to \cref{nv4:app:machine}.

% =============================================================================
% §II  PROPOSED FORMULATION
% =============================================================================
\section{Proposed formulation}
\label{sec:apercu}

\begin{center}\fbox{\parbox{0.93\linewidth}{\small
\begin{problem}\label{prob:central}
\itshape Characterize the geometry of the space of beliefs that a finite agent,
inferring a fixed and noisy world through finite sensors, can inhabit: its
\emph{distances} (the cost of changing one's mind), its \emph{boundary} (certainty) and its
\emph{curvature}.
\end{problem}}}\end{center}

\begin{figure}[!t]
\centering
\includegraphics[width=\linewidth]{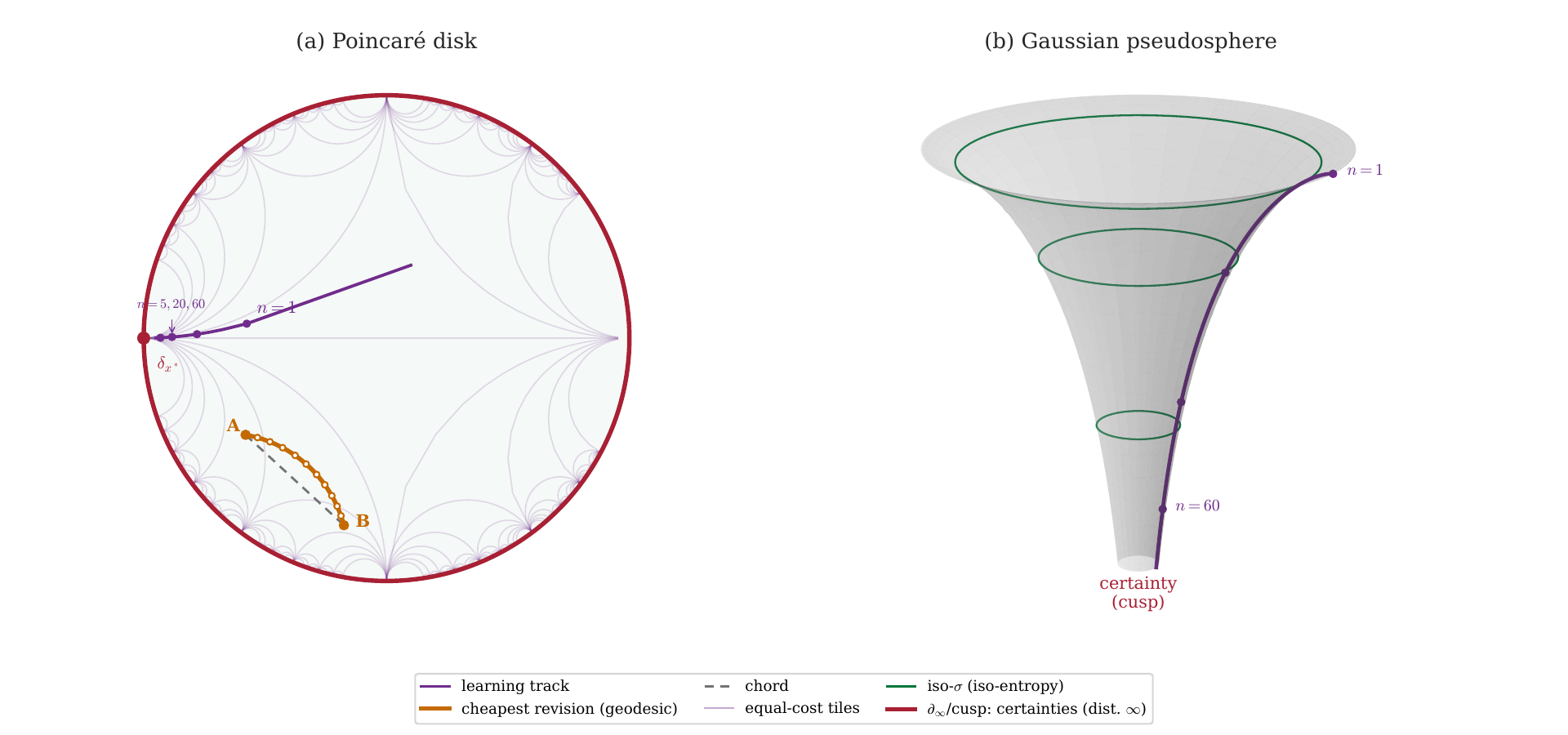}
\caption{\textbf{The cost geometry of belief.} \emph{(a)} The location-scale leaf (a single location-scale family, a 2-D submanifold of $\Ptwo$) in the
\emph{Poincaré disk} (the conformally equivalent rendering of the $\{\sigma>0\}$ half-plane used in
the proofs), drawn at the eikonal limit $e=0$, where uniform pricing (\cref{nv4:ax:P1}) makes the leaf
exactly hyperbolic. The faint backdrop is the modular tiling: at $e=0$ the constant-curvature metric
makes its tiles \emph{congruent} (equal cost-area): a uniform cost grid; they crowd toward the
boundary $\sigma{=}0$ (the Diracs, $\partial_\infty$), which is thus infinitely many tiles away,
at infinite distance. The highlighted geodesic joins two beliefs $A$ and $B$; uniform pricing makes its beads
\emph{equal-cost steps} (the cheapest revision \emph{opens up}; the straight chord, dashed, costs
more). Along the \emph{trajectory} the posterior concentrates with evidence: the fixed beliefs
labelled by evidence count $n=1,5,20,60$ rest ever nearer the boundary (specifically the target Dirac
$\delta_{x^\star}$, the certainty at the true state), their steps shrinking without ever reaching it
(the standardized base then nearer Gaussian, \cref{nv4:rem:bvm}). At $e=0$ the leaf's (intrinsic) Gauss curvature is constant, $K=-1/(2cJ_0)$ with $J_0$ the Fisher information of the unit-variance base ($-1/(4J_0)$ at the
representative $c{=}2$, a unit image, \cref{nv4:thm:jauge}); for $e>0$ it varies
(\cref{nv4:thm:hyper}). \emph{(b)} The same leaf as a Gaussian \emph{pseudosphere} (local model,
Hilbert's theorem \citep{docarmo1976}): the width equals the spread $\sigma$ (iso-entropy circles),
and the same fixed beliefs $n=1,\dots,60$ stand ever nearer the cusp ($H\to-\infty$): in both panels
the disk boundary and the cusp coincide. Off-figure,
among location-scale leaves the Gaussian ($J_0{=}1$, Stam) is the most curved
(\cref{nv4:thm:locscale}). All quantities are verified numerically.}
\label{fig:intuition}
\end{figure}

\noindent Our formulation is a \emph{geometry of belief-cost}: the cost of changing one's mind is a
\emph{distance}, the effort to revise one belief into another, measured in \emph{nats}, the
natural unit of information ($\ln$; one nat $\approx1.44$ bits).

\noindent \emph{We pose two postulates} (made precise in \cref{sec:cadre}):
\textbf{P0} \emph{(the arena)}: a revision \emph{cost} is a scalar price on optimal transport, so
beliefs live in the Wasserstein space $(\Ptwo,\Wtwo)$, $\Ptwo$ the densities of finite spread (the
possible posteriors), $\Wtwo$ the optimal-transport distance between two beliefs (the least cost of
morphing one into the other), whose infinitesimal ruler is the transport metric $g_{\Wtwo}$
\citep{otto2001,ambrosio2008}; and \textbf{P1} \emph{(uniform information pricing)}: the price is \emph{eikonal}, one
nat of knowledge costing the same length everywhere. \emph{Their consequence} is a single cost-metric \emph{form}, the transport metric conformally reweighted by the Fisher information:
\[
  \boxed{\;\tilde g_{e,U}\;=\;2\,(e+U)\,g_{\Wtwo}\,,\qquad U=cJ\,,\qquad
  J=\lVert\nabla_{\Wtwo}H\rVert^2=\int\frac{|\nabla p|^2}{p}\;}
\]
where $p$ is the belief (a density over states $x\in\Rbb^n$, $n$ the state dimension), $H=-\int p\ln p$ its \emph{entropy} (in
nats) and $-H$ its knowledge; $\tilde g_{e,U}$ is the resulting \emph{cost metric} (the local ruler
giving the price of a small revision, namely the transport ruler $g_{\Wtwo}$ reweighted point by point
by $2(e+U)$) with \emph{relief} $U\ge0$ the local price of precision (steepening where the belief sharpens), $e\ge0$ a baseline, and $c>0$
a fixed \emph{cost unit} (only cost \emph{ratios} mean anything; $c=2$ below is one representative).
The two gradients differ: $\nabla p$ is the ordinary spatial gradient, so $J=\int|\nabla p|^2/p$ is the
classical \emph{Fisher information} (large where the belief is sharp, infinite at certainty), whereas
$\nabla_{\Wtwo}H$ is the gradient of $H$ taken \emph{along transport} (Otto's calculus on $\Ptwo$),
its squared length equal to $J$ by the \emph{Otto identity} \citep{otto2001,villani2009}. That
identity is the key step: along transport the Fisher information becomes the \emph{squared slope} of $-H$
(its \emph{recasting} from a component of the distinguishability ruler in Fisher--Rao) and this lets uniform pricing (P1)
\emph{characterize} $U=cJ$.

\emph{From the two postulates the characterization follows}, on the conformal class
(certainty being the limit where the belief sharpens to a point):
\begin{itemize}
\item an \textbf{infinite-distance boundary} (\cref{nv4:thm:carI}): a well-posed inference (existence and stability; uniqueness left open) rejects certainty to
  \emph{infinite} cost-distance as soon as the relief dominates the Fisher information (sufficiency in
  every dimension; necessity theorem-grade on the power family $\{cJ^\alpha\}$ (power-law reliefs), the general case
  conjectured, \cref{nv4:rem:C1});
\item a \textbf{uniformly-priced family} (\cref{nv4:thm:carII}): the eikonal price is \emph{equivalent}, among continuous reliefs, to
  $U=cJ$, the Fisher family;
\item a \textbf{rigidity} (\cref{nv4:thm:hyper,nv4:thm:locscale}), essentially
  \emph{location-scale} (a fixed base shape positioned by a centre $\mu$ and scaled by a spread $\sigma$): these geometries are
  \emph{hyperbolic} (negative Gauss curvature; at $e=0$ constant, $K=-1/(2cJ_0)$), where $J_0$ is the Fisher information
  of the family's \emph{standardized shape} (the variance-$1$ base; distinct from the precision
  $J=J_0/\sigma^2$); the Stam bound $J_0\ge1$ \citep{stam1959} (equality iff Gaussian) ranks the \emph{Gaussian} as
  the most hyperbolic location-scale belief (at $e=0$) ($-\tfrac14$ being its image under a unit).
\end{itemize}

\begin{center}\fbox{\parbox{0.92\linewidth}{\small
\textbf{Scope of the characterization.} Under P0 and P1, among continuous reliefs, the Fisher family
$U=cJ$ is the unique uniform-pricing cost (\cref{nv4:thm:carII}). Removing either postulate leaves the
selection open. Necessity beyond power laws, and beyond the conformal class, stays conjectural
(\cref{nv4:rem:C1}).
}}\end{center}

\noindent These split along the two roles of the entropy: its \emph{value} $-H$ (how concentrated the
belief is) sets the boundary and the cost floor, while the base \emph{slope} $\sqrt{J_0}$ (the family's
shape) sets the $e=0$ curvature $K=-1/(2cJ_0)$, so the Gaussian's extremality (a $J_0$ fact) is
orthogonal to the boundary (a $-H$ fact). The relief itself separates a \emph{global} exponent from a
\emph{varying} shape: in $U=cJ^\alpha$ the exponent is fixed once for the whole geometry (uniform pricing
$\Rightarrow\alpha{=}1$), whereas the shape enters only through $J_0$ in $J=J_0/\sigma^2$ and varies
from belief to belief (ranked by Stam); the exponent is structural, the shape the degree of
freedom. Whence a \emph{cost floor} for reaching a given precision, diverging at certainty, and a
\emph{relativity of cost}: only relative cost means anything, the value $-\tfrac14$ being one image
under a change of unit (\cref{nv4:thm:jauge}); the nat fixes the information axis (Shannon entropy
is unique up to a unit, \cref{nv4:ax:mesure}), the gauge $c$ fixes the cost axis, and their
exchange rate (the price of one nat at $e=0$: $\sqrt{2c}$ lengths) is conventional; a global
change of state units leaves the $e=0$ geometry unchanged (\cref{nv4:prop:etat-unites}). Thermodynamics anchors the cost unit (one nat is worth $\kB T$, the thermal energy scale) and \emph{motivates} the picture; the results are geometric, in nats.

\noindent \cref{fig:intuition} reads this on a single location-scale leaf: panel \emph{(a)} as the
\emph{Poincaré disk} (the conformal image of the $\{\sigma>0\}$ half-plane), the boundary at
$\sigma{=}0$, the cheapest revision \emph{opening up} (widening the spread), and panel
\emph{(b)} as a pseudosphere whose width is the spread $\sigma$.

\noindent We call these geometries \emph{belief-cost geometry} and keep the term throughout. Two
distinctions fix its place: here \emph{belief} is a Bayesian density (a posterior over a fixed
system), set apart from the Dempster--Shafer belief function and the decision-theoretic Bayes cost; and
\emph{cost} is the length of a belief revision, set apart from distinguishability, dissipation length,
and free-energy-principle geometries (\cref{sec:travaux}). A direct algorithmic reading (a Kalman
filter revising at finite cost \emph{vs} a network frozen on a certainty, which must be retrained)
makes the boundary concrete (\cref{enc:ml}).

% =============================================================================
% §III  SETTING
% =============================================================================
\section{Setting}
\label{sec:cadre}\label{nv4:sec:objet}

We first state the question, then build the framework in layers of decreasing necessity: the
\emph{object} (forced by finiteness), the two \emph{postulates} (the posed and defended choices),
and the remaining \emph{assumptions} (transparent, a technical restriction, or borrowed). \Cref{nv4:tab:entropy}
gathers every item with its status and its entropic reading. Each statement
below carries this status as a tag: \emph{(setting)} for the modeling frame, \emph{Postulate} for a
posed choice, and \textbf{B}$n$ for a borrowed tool; the rest is proved here. Provenance is further
labelled \emph{tier-P} (physical/interpretive: motivates, never a proof), \emph{tier-E}
(borrowed: a standard theorem cited), or \emph{tier-D} (proved here, in pure metric geometry);
the full dependency map is \cref{nv4:tab:dag}.\footnote{The letter $n$ carries two
context-disambiguated meanings: the ambient state-space dimension ($\Rbb^n$) and the
evidence count ($n=1,5,\dots$, \cref{fig:probleme}).}

\subsection{Definitions and setting assumptions}

The setting is the one fixed in \cref{fig:probleme}: a finite agent infers a fixed hidden
state $x$ through a finite, noisy sensor, and its belief is the Bayes posterior, a density
over the states. We read it statically: a posterior over the fixed state $x$ at a given
instant. Certainty is excluded, the barred Dirac of \cref{fig:probleme}: the two finitudes
of \cref{sec:intro}, bounded information and bounded energy, keep it out of reach. It is
exactly the boundary where the Fisher information $J$ diverges.

Bayes \emph{motivates} this object, the coherent epistemic state of a finite agent
\citep{cox1946,jaynes2003}; it enters none of the proofs below, which use only that the
belief is a regular density (\cref{nv4:ax:objet}). We study the cost geometry of
this posterior: the price of revising it.

\begin{definition}[belief, entropy, Fisher; the carrier set]\label{nv4:def:arene}
A \emph{belief} is a density $p\in\Ptwo(\Rbb^n)$, the probability measures of finite
second moment $m_2:=\int|x|^2 p$; its entropy (Shannon, nats) is $H(p)=-\int p\ln p\,\dd x$, its
knowledge $-H(p)$, its Fisher information
$J(p)=\int|\nabla p|^2/p\in[0,+\infty]$. We write $\Ent:=-H$ and
$\partial_\infty:=\{H=-\infty\}$ the \emph{boundary of certainties} (singular
measures, including all Diracs). The carrier $\Ptwo(\Rbb^n)$ is forced by finiteness
(the object); the \emph{metric} that turns it into the cost \emph{arena} (the
Wasserstein distance $\Wtwo$) is the first modeled choice (the metric, \cref{nv4:ax:P0}), its
relief fixed by the second (uniform pricing, \cref{nv4:ax:P1}), not posited here. Gaussian dictionary ($n=1$):
$H=\tfrac12\ln(2\pi\sigma^2)+\tfrac12$, $J=1/\sigma^2$.
\end{definition}

\begin{definition}[regularity class \emph{\small(setting)}]\label{nv4:def:R}
$\mathcal R\subset\Ptwo(\Rbb^n)$: $C^1$ densities, strictly positive, with finite $m_2$
and $J$. It is the interior domain; $\partial_\infty$ is excluded from it.
\end{definition}

\begin{premisse}[\textbf{A1}: object \emph{\small(setting)}]\label{nv4:ax:objet}
The belief is a density of class $\mathcal R$ (\cref{nv4:def:R}), a point of
the arena. Finiteness \emph{closes off the Dirac} ($\partial_\infty$ excluded: infinite
precision would make $J$ infinite there); the continuous $C^1$ class is a working
idealization.
\end{premisse}

\begin{remarque}[two operational constraints \emph{\small(setting)}]\label{nv4:rem:motive}
Two constraints drive the characterization. \emph{(a) Boundary excluded:}
$\partial_\infty$ (infinite Fisher information, \cref{nv4:def:arene}) is out of
budget, the belief lives in the interior ($\mathcal R$). \emph{(b) Well-posed:}
inference (minimization over the arena) must admit an \emph{interior}
minimizer, with no collapse toward $\partial_\infty$. \emph{(The boundary is costly
because $J$ diverges there; $J$ is here the slope of $-H$, not the Fisher--Rao
tensor, \cref{nv4:rem:dualJ}.)}
\end{remarque}

\begin{definition}[plausibility and relief \emph{\small(setting)}]\label{nv4:def:plausibilite}
The plausibility of a belief (no longer the posterior $p(x\mid y)$ over the states of the world, but a degree over the beliefs themselves) is modeled by a Gibbs measure on the space
of beliefs \citep{jaynes2003},
\[
  \Pi\ \propto\ \nu\,\exp(-U),\qquad U=-\ln\frac{\dd\Pi}{\dd\nu}\ \ge\ 0 ,
\]
$\nu$ being a reference measure. The \emph{relief} $U$ is the negated
log-density: an \emph{improbability} (least commitment). The boundary of
certainties ($-H=+\infty$) is all the less plausible the more the relief
grows there; no form of $U$ is fixed here.
\end{definition}

\noindent$\Pi$ is a measure on the space of beliefs: a density, not an
evolution. The relief $U$ remains \emph{generic}: lower semi-continuous and $\ge0$.

\begin{premisse}[\textbf{A2}: measure \emph{\small(setting)}]\label{nv4:ax:mesure}
Knowledge is measured by an \emph{information} measure; among those satisfying continuity,
maximality on the uniform, expansibility and grouping (recursivity), the Shannon entropy is the
\emph{unique} choice up to a unit (Shannon--Khinchin \citep{shannon1948,khinchin1957}). Operationally $-H$ is the
description cost the source-coding theorem makes inevitable \citep[thm~4.1, p.~78, with the asymptotic equipartition property (AEP) p.~80]{mackay2003}. This is a
\emph{discrete} characterization; for continuous beliefs it does not transfer literally (the
differential entropy is neither nonnegative nor reparametrization-invariant), but it motivates
reading knowledge on the differential entropy $H=-\int p\ln p$ of \cref{nv4:def:arene}, knowledge
being $-H$, defined up to an additive constant (a reference scale) that drops out of every
result, since these use only the entropy \emph{slope} $\nabla H$ (this additive freedom is distinct
from, and orthogonal to, the \emph{multiplicative} cost gauge of \cref{nv4:thm:jauge}). The
Shannon--Khinchin axioms make this the natural reading once one elects to measure knowledge by an
information measure at all, itself a framing choice: the measure is thus \emph{transparent and
motivated} rather than merely imposed, though, unlike the object forced by finiteness, the axioms
(notably grouping) remain a posed normative choice (relax grouping $\to$ Rényi).
\emph{(Independently and downstream: Shannon is the entropy for which the Otto identity
\cref{nv4:ax:B1} makes the transport slope $\sqrt J$; this coupling to the arena is specific to it.)}
\end{premisse}

\subsection{The two postulates}
\label{nv4:sec:postulats}

What is costly to the agent is to \emph{change} belief. To turn the excluded certainty into
distances we make \emph{two} posed, defended commitments: the \emph{arena} (Postulate~0) and
\emph{uniform pricing} (Postulate~1), nested. Each is a choice an informed reader could refuse, and each is
defended; together they pin the cost geometry, while every other ingredient is forced, transparent, or
borrowed.

\begin{postulat}[the cost is a scalar price on optimal transport]\label{nv4:ax:P0}
We model the \emph{cost} of changing a belief as a \emph{scalar price field over the optimal
transport} of probability mass on the world's state space, at fixed mass: a transport metric,
reweighted by a scalar relief. Two clauses, both contestable: the cost is a \emph{transport}
(mass moved over the state space, not a reweighting in memory), and the price is \emph{scalar}
(isotropic: it depends on the belief, not on the direction of the revision). This is a statement
about \emph{cost}, not mechanism: a belief is \emph{updated} by Bayes (a multiplicative
reweighting), whereas the \emph{price} of any revision is read as a transport length. The
non-conformal (anisotropic) alternative is left open (\cref{nv4:rem:C1}).
\end{postulat}

\noindent From this commitment the arena \emph{follows}. Mass
conservation (a belief is a probability, total mass one), locality (the total cost is
\emph{additive} over a ground cost $c_0(x,y)$ on elementary state displacements, the Kantorovich
form $\int c_0\,\dd\pi$, excluding coupling functionals non-additive in the plan
\citep{santambrogio2015,villani2009}) and the demand for a length structure single out an
optimal-transport metric; a power ground cost $d(x,y)^p$ then pins the Wasserstein family $W_p$.
The quadratic case $p=2$ is fixed by a single \emph{intrinsic} requirement, prior to any result
below: to speak of geodesics, a boundary and a curvature \emph{at all}, the length must be
\emph{Riemannian}: a differentiable kinetic action $\tfrac12\int\md\gamma^2$ (Benamou--Brenier, \citealp{benamou2000})
admitting the Maupertuis--Jacobi correspondence \citep{arnold1989} on which the machinery of \cref{nv4:sec:carac}
rests. Among the $W_p$, only $W_2$ is Riemannian \citep{otto2001,ambrosio2008,gigli2012}; $W_{p\neq2}$ is
a metric but carries no such differential structure, standard optimal-transport folklore, with the
necessity direction ($p\neq2$) made exact on the location-scale leaf (a single location-scale family, a 2-D \emph{leaf} of $\Ptwo$) in \cref{nv4:prop:quad}
(failure already on a single $2$D leaf rules out a Riemannian structure on the whole space; the
positive global Riemannian structure of $W_2$ stays the borrowed fact above). This fixes $p=2$ on its own terms: \emph{before}, and independently of,
the downstream payoffs it happens to unlock (the Otto slope $\|\nabla_{\Wtwo}H\|^2=J$,
\cref{nv4:ax:B1}; the flat location-scale base, \cref{nv4:ax:B4,nv4:ax:B7}). For a digital twin,
whose role is to \emph{represent} the world, we locate this cost in the \emph{represented} space
(the states), not in the machine's memory, itself part of the same commitment, not a deduction
from the twin's role; the parameter-space reading (elastic weight consolidation, \cref{enc:ml}) is
the explicit alternative.

\begin{proposition}[the quadratic exponent is forced]\label{nv4:prop:quad}
Within the transport family, $p=2$ is the unique exponent for which the cost geometry is
\emph{Riemannian} (and, heuristically, the one for which the slope of $-H$ is the classical $L^2$ Fisher
information). On any location-scale family with standardized base $\varphi_0$ (quantile $Q_0$), the
quantile isometry (\cref{nv4:ax:B7}; the general-$p$ $W_p$/$L^p$ quantile identity, \citep[\S2]{santambrogio2015}) makes the $W_p$ length element of a tangent $(d\mu,d\sigma)$
the $L^p$ norm of $d\mu\,\mathbf 1+d\sigma\,Q_0$ on the plane
$\mathrm{span}\{\mathbf 1,Q_0\}\subset L^p(0,1)$ (well-defined when the base has a finite $p$-th
moment, $Q_0\in L^p$; automatic at $p=2$, where finite variance suffices). By the Jordan--von
Neumann theorem \citep{jordanvn1935} this norm derives from an inner product iff the parallelogram
law holds on $\mathrm{span}\{\mathbf 1,Q_0\}$, which for any non-degenerate base ($Q_0$ not a.e.\
constant) occurs \emph{iff} $p=2$, where it equals $d\mu^2+d\sigma^2$, the flat half-plane (\cref{nv4:ax:B4,nv4:ax:B7}). For $p\neq2$ it is genuinely Finsler. On the $W_2$ arena the same exponent is what makes the tangent norm the $L^2$ kinetic energy
$\int p\,|v|^2$, whose pairing with the $L^2$ score gives the Otto identity
$\int p\,|\nabla\ln p|^2=J$ (\cref{nv4:ax:B1}, a separate global fact); heuristically, for $p\neq2$
the corresponding slope of $-H$ is an $L^q$ object ($\tfrac1p+\tfrac1q=1$) rather than the classical
$L^2$ Fisher. The Riemannian length required by the transport postulate (\cref{nv4:ax:P0}; Maupertuis)
thus forces $p=2$ (the parallelogram law on $\mathrm{span}\{\mathbf 1,Q_0\}$ holds iff $p=2$;
checked numerically for the Gaussian, Laplace, and a skewed base in the verification companion).
\end{proposition}

\noindent This fixes the \emph{exponent} within the transport family; that the cost is a transport
\emph{at all} remains \cref{nv4:ax:P0}, the one posed choice (the parameter-space alternative,
\cref{enc:ml}, is a Fisher--Rao cost on weights). It is orthogonal to \cref{nv4:rem:C1}.
Dynamically the canonical gradient-flow structure of the heat equation lives in the same $W_2$
(the JKO theorem \citep{jko1998}: heat flow $=$ $W_2$ gradient flow of Shannon entropy): a
tier-P corroboration that enters no proof (our reading is static).

\begin{premisse}[\textbf{A3}: arena \emph{\small(consequence of \cref{nv4:ax:P0})}]\label{nv4:ax:arene}
The geometry deforms \emph{transport}: arena $(\Ptwo,\Wtwo)$, the two-step consequence of
\cref{nv4:ax:P0} (transport \emph{type}; then $p=2$ for a Riemannian length, \cref{nv4:prop:quad}; cf.\
\cref{nv4:rem:dualJ,nv4:rem:amari}).
\end{premisse}

\noindent\emph{Physical corroboration (motivates, never a proof).} The same arena is
\emph{also reached} from physics: driving a density from $p_0$ to $p_1$ in time $T$ dissipates at
least $\Wtwo^2(p_0,p_1)/T$ \citep{benamou2000,dechant2019}, so the physical price of a revision is
itself a $\Wtwo$ length (\cref{sec:travaux}). This is \emph{not} an independent derivation, and
\emph{no proof depends on the physical reading}:
invoking the dissipation floor for a \emph{belief} assumes the very identification of
\cref{nv4:ax:P0} (it relocates the contestable step, it does not remove it). It stays tier-P;
it \emph{motivates}, never proves (\cref{nv4:tab:dag}).

\medskip
\noindent The arena and its scalar price still leave the relief $U$
free: \emph{every} relief that dominates the Fisher information sends $\partial_\infty$ to infinite distance
(\cref{nv4:thm:carI}), so $cJ$, $cJ^2$, or $J+\mathrm{const}$ would all serve; Fisher is not yet
singled out. One further commitment singles it out.

\begin{postulat}[uniform information pricing]\label{nv4:ax:P1}
The price is \emph{uniform}: one nat of knowledge costs the same length everywhere: the metric
slope of knowledge $-H$ is constant across the space of beliefs. This is a choice an agent could
refuse (a distorted price would charge confidence unevenly); we adopt it as the calibration a
digital twin owes its own uncertainty. By the characterization \cref{nv4:thm:carII} it singles out,
\emph{among continuous reliefs}, a \emph{unique family} (up to the unit $c$), $U=cJ$ (the Fisher family), and the boundary at infinite distance then follows
($cJ\in\Ccal$). Its denials are live: $U=J+\mathrm{const}$ dominates but prices nats unevenly;
$U=cJ^{\alpha}$, $\alpha>1$, sends $\partial_\infty$ to infinite distance but prices nats unevenly.
\end{postulat}

\noindent \emph{Alternative pricings.} Other calibrations are available, a price proportional to
energy ($U\propto E$) or to a divergence ($U\propto\mathrm{KL}$); we adopt uniform pricing and leave
these open.

\noindent P1 \emph{presupposes} P0: the constant-slope condition is read on the transport metric
(the slope is $\sqrt J$ by the Otto identity \cref{nv4:ax:B1}, which lives on the arena). The two
postulates are \emph{nested}; P0 builds the arena and its scalar
price, P1 fixes that price. The uniform-pricing reading aligns with Landauer (each nat dissipates the same
$\kB T$); like the dissipation floor, this is \emph{defense, not proof}: it enters no theorem
(\cref{nv4:thm:carII} uses only the eikonal condition as a hypothesis; \cref{nv4:prop:firewall}).

\begin{principe}[the recasting of $J$: from ruler component to slope]\label{nv4:rem:dualJ}
The whole construction turns on a single recasting. In information geometry $J$ is a
\emph{component} of the Fisher--Rao metric tensor (\v{C}encov), the ruler of distinguishability: the
Fisher--Rao squared-length of the translation generator $\delta p=-\nabla p$
($\int(\delta p)^2/p=\int|\nabla p|^2/p=J$). Here, by the
Otto identity $J=\|\nabla_{\Wtwo}H\|^2$ (\cref{nv4:ax:B1}), the same scalar becomes the \emph{squared slope
of $-H$} on the transport arena. This recasting is what lets uniform pricing \emph{characterize} $U=cJ$
(\cref{nv4:thm:carII}) rather than restate it, and it makes the \emph{forbidden} object
($J=\infty$ at the Dirac) and the \emph{priced} one ($U=cJ$) two readings of one divergence:
the slope of $-H$ at the boundary. It recurs in the eikonal
(\cref{nv4:thm:carII}) and the curvature that reads $J_0$ (the Fisher information of the standardized base, $J_0:=\int(\varphi_0')^2/\varphi_0$, $=1$ for the Gaussian; \cref{nv4:thm:locscale}), off which the
positioning against \v{C}encov is then read (\cref{nv4:rem:amari}); the boundary at infinite distance is the exception; it
needs only \cref{nv4:ax:B2}, not the recasting.
\end{principe}

\subsection{The cost class}

This conformal class is the setting in which every result below is established. \cref{nv4:ax:P0} prices a
revision by a \emph{scalar} (one number at each belief) and a scalar multiplying the transport
metric is exactly a \emph{conformal} reweighting: $\tilde g=2(e+U)\,g_{\Wtwo}$, the ``$2$'' being the kinetic normalization of the Maupertuis--Jacobi action (\cref{nv4:thm:maupertuis}); the quadratic transport exponent $p=2$ is forced separately (\cref{nv4:prop:quad}). The conformal form is therefore the geometric face of ``scalar price''. The three results are then three statements about \emph{which} reliefs $U$ make
this geometry well-posed; the candidate class $\Ucal$ gathers all admissible reliefs.

\begin{definition}[space of candidates \emph{\small(setting)}]\label{nv4:def:candidat}
A \emph{relief} is a function $U:\Ptwo(\Rbb^n)\to[0,+\infty]$. The \emph{space
of candidates} is
\[
  \Ucal:=\bigl\{\,U:\Ptwo(\Rbb^n)\to[0,+\infty]\ \big|\ U\ \text{is }
  \tau\text{-sequentially l.s.c.}\,\bigr\},
\]
$\tau$ being the narrow convergence. To each $U\in\Ucal$ and each level
$e\ge0$ is associated the \emph{cost geometry} ($2$: unit; $e>0$ for
licit reparametrization and separation of points, $e=0$ for the eikonal/limiting case,
conformal metric on $\{U>0\}$)
\begin{equation}
  \tilde g_{e,U}:=2\,(e+U)\,g_{\Wtwo},\qquad
  \ell_e(\gamma):=\int_0^T\!\sqrt{2(e+U(\gamma_t))}\,\md\gamma(t)\,\dd t,\qquad
  d_e:=\inf\ell_e .
  \label{nv4:eq:cout}
\end{equation}
The candidate class is \emph{conformal}: the scalar-price form posited in \cref{nv4:ax:P0};
non-conformal (anisotropic) deformations are the open horizon (\cref{nv4:rem:C1}).
\end{definition}

\noindent The results then carve, inside $\Ucal$, three nested subclasses
$\Fcal\subset\Ccal\subset\Wcal\subset\Ucal$ (the Fisher family, the well-posed and the boundary classes;
defined below in \cref{nv4:def:fisher,nv4:def:classe,nv4:def:mur}; \cref{nv4:fig:univers}).

\section{Characterizations}
\label{sec:carac}\label{nv4:sec:carac}

We first gather the borrowed tools (\cref{nv4:sec:tools}) and the machinery valid
for every candidate (\cref{nv4:sec:machine}), then establish the three characterizations:
the \emph{boundary} ($\partial_\infty$), the \emph{uniformly-priced family} (the eikonal $U=cJ$), and the
\emph{curvature} (hyperbolicity and Stam rigidity). Throughout, P0 and P1 are taken as given, and we
derive their consequences.

\noindent The classes and the entropy reading used throughout are gathered at a glance in \cref{nv4:fig:univers,nv4:tab:entropy}.

\begin{figure}[t]
\centering
\begin{tikzpicture}[
  x=1cm,y=1cm,>=Stealth,font=\scriptsize,line width=0.6pt,
  gate/.style={draw=soudure!80!black,fill=soudure!12,text=soudure!82!black,
    rounded corners=2pt,inner sep=2.6pt,align=center,font=\scriptsize},
  nm/.style={inner sep=2.2pt,font=\bfseries\scriptsize,anchor=west,
    fill=white,rounded corners=1.5pt},
  ex/.style={font=\scriptsize\itshape,text=black!50,align=center,inner sep=1.8pt,
    fill=white,rounded corners=1.5pt},
  P/.style={draw=axenergie!85!black,fill=axenergie!9,text=axenergie!82!black,
    rounded corners=2pt,align=center,font=\bfseries\scriptsize,inner sep=3pt,text width=58mm},
  motiv/.style={->,dashed,line width=0.8pt,draw=axenergie!72!black},
]
% anneaux concentriques (externe -> interne)
\filldraw[rounded corners=5pt,fill=black!3,draw=black!55] (0.2,0.3) rectangle (13.8,11.2);
\filldraw[rounded corners=5pt,fill=axinfo!5,draw=axinfo!68!black] (1.0,1.0) rectangle (13.0,9.4);
\filldraw[rounded corners=5pt,fill=axinfo!10,draw=axinfo!88!black] (1.9,1.8) rectangle (12.1,7.2);
\filldraw[rounded corners=5pt,fill=okD!8,draw=okD!75!black] (3.0,2.6) rectangle (11.0,4.9);
% ring names (white background -> readable over edges and arrows)
\node[nm,text=black!65] at (0.45,10.7){$\Ucal$: all admissible reliefs ($U$ l.s.c.\ $\ge0$)};
\node[nm,text=axinfo!68!black] at (1.2,9.05){$\Wcal$: certainty at infinite distance};
\node[nm,text=axinfo!88!black] at (2.05,6.95){$\Ccal$: well-posed (existence/stability)};
\node[nm,text=okD!72!black] at (3.2,4.6){$\Fcal=\{cJ\}$: uniform};
% gates (= characterizations): centered, in the trough of each band
\node[gate,text width=60mm] at (6.9,10.0) {domination \emph{at the boundary} $\Longleftrightarrow$ $\partial_\infty$ at infinite distance\ (\textbf{Char.\,I}\,ii, on $\{cJ^\alpha\}$)};
\node[gate,text width=60mm] at (6.9,8.35) {\emph{global} domination $\Longrightarrow$ well-posed\ (\textbf{Char.\,I}\,i, sufficiency)};
\node[gate,text width=66mm] at (6.9,5.55) {uniform pricing (eikonal) $\Longleftrightarrow$ $U{=}cJ$\ (\textbf{Char.\,II}) \;\textperiodcentered\; key step: slope $\sqrt J$ (Otto identity)};
% examples (white background)
\node[ex] at (6.9,7.55) {$U=cJ^\alpha,\ \alpha>1$: $\partial_\infty$ at infinite distance, but not uniform};
\node[ex] at (6.9,6.3) {$U=J+\text{cst}$: dominant, uneven price};
\node[ex] at (6.9,0.7) {outside $\Wcal$: underpricing $\alpha<1$ $\Rightarrow$ the boundary comes to finite distance $\Rightarrow$ collapse};
% interieur F : representant
\node[circle,fill=okD!75!black,inner sep=1.7pt] at (6.9,3.78) {};
\node[font=\bfseries\scriptsize,text=okD!42!black,anchor=south] at (6.9,3.92) {$2J$};
\node[ex,text=okD!48!black] at (6.9,3.4) {every hyperbolic $\Fcal$: $K=-1/(2cJ_0)<0$};
% P (physique) hors univers
\node[P] (Phys) at (7,12.0) {\textbf{P}: finite agent, physical information};
\draw[motiv] (Phys.south) to[out=-90,in=92] (10.6,9.42);
\draw[motiv] (Phys.south east) to[out=-65,in=90] (11.5,7.22);
\draw[motiv] (Phys.south east) to[out=-48,in=88] (10.8,4.92);
\node[font=\scriptsize\itshape,text=axenergie!70!black,align=center] at (12.3,11.55)
  {motivates, never\\proves};
\end{tikzpicture}
\caption{\textbf{The universe of cost geometries.} The nested rings
$\Fcal\subset\Ccal\subset\Wcal\subset\Ucal$: the boundary class $\Wcal$ (\cref{nv4:thm:carI}), the
well-posed class $\Ccal$, the uniformly-priced Fisher family $\Fcal$ (\cref{nv4:thm:carII}),
with the Fisher representative $2J$ at the center, every member hyperbolic; off-figure, among the location-scale leaves of $\Fcal$ the Gaussian is the most hyperbolic (Stam rigidity,
\cref{nv4:thm:locscale}). The physics \textbf{P} \emph{motivates} the boundaries
(dashed) without entering any proof (\cref{nv4:prop:firewall}).}
\label{nv4:fig:univers}
\end{figure}

\begin{table}[!t]
\centering\small
\caption{\textbf{The framework and its geometry of $-H$: one map.} \emph{Top:} every commitment
with its status (exactly two posed postulates; all else forced, transparent, borrowed, or a unit).
\emph{Bottom:} each object re-read as a fact about the entropy $-H$ via the Otto identity (a
presentation, adding no theorem).}
\label{nv4:tab:entropy}
\begin{tabular}{@{}p{0.25\textwidth} p{0.39\textwidth} p{0.24\textwidth}@{}}
\toprule
\multicolumn{3}{@{}l}{\textit{What makes $-H$ a geometry}}\\
Element & Commitment & Status \\
\midrule
Object (\cref{nv4:ax:objet}) & belief $=$ density; certainty ($-H=+\infty$) excluded & forced (finiteness) \\
Measure (\cref{nv4:ax:mesure}) & knowledge $=$ Shannon $-H$ (Shannon--Khinchin) & transparent \\
\textbf{P0} (\cref{nv4:ax:P0}) & cost $=$ scalar price on optimal transport $\Wtwo$ & \textbf{postulate} \\
\textbf{P1} (\cref{nv4:ax:P1}) & uniform price: one nat, one length everywhere & \textbf{postulate} \\
Otto (\cref{nv4:ax:B1}) & the slope of $-H$ is $\sqrt J$: the key step & borrowed \\
Cost class (\cref{nv4:def:candidat}) & conformal (scalar) relief $2(e+U)g_{\Wtwo}$ & restriction (C1) \\
Exponent \& unit ($p{=}2$, $c$) & $p{=}2$ forced by Riemannian length (\cref{nv4:prop:quad}); $c$ a cost gauge (\cref{nv4:thm:jauge}) & forced / unit \\
\addlinespace
\midrule
\multicolumn{3}{@{}l}{\textit{Its geometry: each object read through $-H$}}\\
Geometric object & Entropic reading & Form \\
\midrule
Cost \emph{floor} & entropy-distance floor; $|\Delta H|$ is only the entropy-axis projection (a pure location shift, $\Delta H=0$, still costs) & $d_0\ge\sqrt{2c}\,|\Delta H|$ (eq.\ along $\nabla(-H)$) \\
Fisher information & squared slope of $-H$ (Otto) & $J=\|\nabla_{\Wtwo}H\|^2$ \\
Uniform pricing (eikonal) & constant entropy slope ($\kappa$, \cref{nv4:def:eikonal}) & $\sqrt J/\sqrt{2U}=\kappa\Leftrightarrow U=cJ$ \\
$\partial_\infty$ (certainty) & divergence of the entropy & $\partial_\infty=\{H=-\infty\}$ \\
Curvature (at $e=0$) & inverse of the base entropy-slope$^2$ & $K=-1/(2cJ_0)$ \\
Stam rigidity (static) & at fixed variance: max entropy $=$ min base slope $=$ most-curved leaf & $J_0\ge1$, eq.\ iff Gaussian \\
Thermodynamic bridge & the floor is an energy (Landauer; tier-P) & $d_0\ge(\sqrt{2c}/\kB T)\,E$ (eq.\ along $\nabla(-H)$) \\
\bottomrule
\end{tabular}
\end{table}

\subsection{Borrowed tools}
\label{nv4:sec:tools}

\begin{premisse}[\textbf{B1}: the Otto identity]\label{nv4:ax:B1}
On $(\Ptwo,\Wtwo)$, $-H$ is geodesically convex \citep{mccann1997} and
$\|\nabla_{\Wtwo}H\|^2=J$: the Otto identity \citep{otto2001}, here via
\citet[§10.4]{ambrosio2008}. \emph{This identity converts ``uniform price'' into ``cost $\propto J$''.} \emph{(The equality is used only by the eikonal
(\cref{nv4:thm:carII}); the boundary at infinite cost-distance (\cref{nv4:thm:carI}) needs only the upper-gradient inequality
\cref{nv4:ax:B2}.)}
\end{premisse}

\begin{premisse}[\textbf{B2}: strong upper gradient]\label{nv4:ax:B2}
For every $AC$ (absolutely continuous) curve: $|\Ent(\gamma_t)-\Ent(\gamma_s)|\le\int_s^t\sqrt{J}\,
\md\gamma\,\dd r$, the left-hand side finite and $\Ent\circ\gamma$ absolutely
continuous as soon as the right-hand side is \citep[§§1.2,2.4]{ambrosio2008}.
\end{premisse}

\begin{premisse}[\textbf{B3}: compatibility $(\Wtwo,\tau)$]\label{nv4:ax:B3}
$(\Ptwo,\Wtwo)$ Polish and geodesic \citep[ch.~7]{ambrosio2008};
$\Wtwo$ is $\tau$-l.s.c.; the $\{m_2\le M\}$ are $\tau$-compact.
\end{premisse}

\begin{premisse}[\textbf{B4}: flat Gaussian leaf]\label{nv4:ax:B4}
$\Wtwo^2(\mathcal N(\mu_1,\sigma_1^2),\mathcal N(\mu_2,\sigma_2^2))
=(\mu_1-\mu_2)^2+(\sigma_1-\sigma_2)^2$ \citep{givens1984}: the leaf
$\Gcal=\{\mathcal N(\mu,\sigma^2)\}$ is isometric to the Euclidean half-plane
$\{\sigma>0\}$; McCann's geodesics stay Gaussian there
\citep[ex.~7.3.14]{ambrosio2008}.
\end{premisse}

\begin{premisse}[\textbf{B5}: metric Hopf--Rinow]\label{nv4:ax:B5}
A locally compact length space in which bounded closed balls are compact is
complete and geodesic \citep[thm~2.5.28]{burago2001}.
\end{premisse}

\begin{premisse}[\textbf{B6}: l.s.c.\ dual of Fisher]\label{nv4:ax:B6}
$J(q)=\sup_{\psi\in C_c^\infty}\{-2\int\Delta\psi\,\dd q-\int|\nabla\psi|^2\dd q\}$
for every $q\in\mathcal P(\Rbb^n)$ \citep[lem.~D.45]{dupuis1997}; as a supremum
of narrowly continuous functionals, $J$ is narrowly l.s.c. \emph{(Hence
$J\in\Ucal$ and, for every $c>0$, $cJ\in\Ucal$.)}
\end{premisse}

\begin{premisse}[\textbf{B7}: flatness of 1D location-scale families]\label{nv4:ax:B7}
The quantile map $p\mapsto F_p^{-1}$ is an isometry of $(\Ptwo(\Rbb),\Wtwo)$ onto
the convex cone of inverse cumulative distribution functions in $L^2(0,1)$
\citep{santambrogio2015,bobkov2019}. For a base $\varphi_0$ (density, mean
$0$, variance $1$, quantile $Q_0$), the location-scale family
$\mathcal L_{\varphi_0}=\{p_{\mu,\sigma}=\sigma^{-1}\varphi_0((\cdot-\mu)/\sigma)\}$
is there the affine $2$-plane $\{\mu\mathbf 1+\sigma Q_0\}$, with
$\langle\mathbf 1,\mathbf 1\rangle=1$, $\langle\mathbf 1,Q_0\rangle=0$,
$\langle Q_0,Q_0\rangle=1$: $(\mathcal L_{\varphi_0},\Wtwo)$ is thus
\emph{isometric to the Euclidean half-plane} $\{\sigma>0\}$, with metric
$\dd\mu^2+\dd\sigma^2$. \emph{(Generalizes \cref{nv4:ax:B4}: the Gaussian is the
case $\varphi_0=\mathcal N(0,1)$.)}
\end{premisse}

\noindent\emph{Metric setting.} $\gamma:[0,T]\to\Ptwo$ is $AC^2$ if there exists
$m\in L^2$ with $\Wtwo(\gamma_s,\gamma_t)\le\int_s^t m$; $\md\gamma$ exists a.e.\
\citep[thm~1.1.2]{ambrosio2008}. We set $\Ecal(\gamma):=\sup_{\text{part.}}\sum_i
\Wtwo^2(\gamma_{t_i},\gamma_{t_{i+1}})/(t_{i+1}-t_i)$, the curve's metric \emph{energy}.

\begin{lemme}[kinetic action (metric energy)]\label{nv4:lem:K}
\emph{(a)} $\gamma\in AC^2\Rightarrow \Ecal(\gamma)=\int_0^T\md\gamma^2$.
\emph{(b)} $\Ecal(\gamma)<\infty\Rightarrow\gamma\in AC^2$.
\emph{(c)} $\Ecal$ is $\tau$-pointwise l.s.c.
\end{lemme}
\begin{proof}
\emph{(a)} ($\le$, i.e.\ $\Ecal\le\int\md\gamma^2$) Cauchy--Schwarz on each subinterval,
$\Wtwo^2(\gamma_{t_i},\gamma_{t_{i+1}})\le\allowbreak(t_{i+1}-t_i)\int_{t_i}^{t_{i+1}}\md\gamma^2$,
sum and then take the sup. ($\ge$, i.e.\ $\int\md\gamma^2\le \Ecal$) the dyadic
quotients $f_N\to\md\gamma$ a.e.\ (a.e.\ limit, \citep[thm~1.1.2]{ambrosio2008}, not
the Cauchy--Schwarz bound); Fatou: $\int\md\gamma^2\le\liminf\int f_N^2\le \Ecal$.
\emph{(b)} the inequality $\frac{(x+y)^2}{\alpha+\beta}\le\frac{x^2}{\alpha}+
\frac{y^2}{\beta}$ yields additivity of $\Ecal$, hence absolute continuity and then
$\int\md\gamma^2\le \Ecal$. \emph{(c)} each term is $\tau$-l.s.c.\
(\cref{nv4:ax:B3}); a sup of l.s.c.\ is l.s.c.
\end{proof}

\begin{lemme}[entropy, slope: recall of imports]\label{nv4:lem:slope}
With $H$, $J$ in the sense of \cref{nv4:def:arene}: $\Ent>-\infty$ on $\Ptwo$, geodesically convex \citep{mccann1997}; its metric
slope equals $\sqrt J$ (\cref{nv4:ax:B1}); it is a strong upper gradient
(\cref{nv4:ax:B2}).
\end{lemme}
\subsection{What holds for every candidate: existence and geodesics}
\label{nv4:sec:machine}

This entire section holds for \emph{every} $U\in\Ucal$: it is the common machinery of
the whole class, independent of the choice of relief. For $\gamma\in AC^2$ joining
$p_0$ to $p_1$: $\Acal_T(\gamma):=\int_0^T[\tfrac12\md\gamma^2+U(\gamma_t)]\dd t$,
$\Phi(T):=\inf\Acal_T$, $\Psi(e):=\inf\ell_e=d_e$. The statements below hold
for every candidate; their proofs, with their regime clauses, are
collected in \cref{nv4:app:machine}.

\begin{lemme}[AM--GM]\label{nv4:lem:amgm}
$\forall a,b\ge0:\ \tfrac12a^2+b\ge\sqrt{2b}\,a$, with equality iff $a=\sqrt{2b}$.
\end{lemme}

\noindent This inequality converts the action $\tfrac12|\dot\gamma|^2+(e+U)$ into
length $\sqrt{2(e+U)}\,|\dot\gamma|$, with saturation at the Jacobi speed.

\begin{theoreme}[Maupertuis--Jacobi correspondence \citep{arnold1989}]\label{nv4:thm:maupertuis}
$\forall U\in\Ucal\ \forall e>0$, writing $p_0,p_1\in\{U<\infty\}$:
\emph{(i)} $\forall\gamma\in AC^2:\ \Acal_T(\gamma)\ge\ell_e(\gamma)-eT\ge\Psi(e)-eT$;
\emph{(ii)} if $\hat\gamma$ minimizes $\ell_e$ ($\Psi(e)<\infty$, $U(\hat\gamma)<\infty$
a.e.), its constant-energy reparametrization $\md{\tilde\gamma}=\sqrt{2(e+U)}$
exists, with parameter $T_e=\int\md{\hat\gamma}/\sqrt{2(e+U)}$, and minimizes $\Acal_{T_e}$
with $\Phi(T_e)=\Psi(e)-eT_e$ \emph{(the bound $L/\sqrt{2e}$ requires $e>0$: this is
what makes the reparametrization licit)};
\emph{(iii)} if $\gamma^\star$ minimizes $\Acal_T$ and saturates \emph{(i)}, then
$\tfrac12\md{\gamma^\star}^2=e+U(\gamma^\star)$ a.e.
\end{theoreme}

\begin{theoreme}[existence of geodesics]\label{nv4:thm:exist}
$\forall U\in\Ucal$, if $\Phi(T)<\infty$, then $\Acal_T$ admits a minimizer in
$AC^2([0,T];\Ptwo)$ joining $p_0$ to $p_1$. \emph{(Same for $\ell_e$, $e>0$,
\citep[§4.3]{buttazzo1998}.)}
\end{theoreme}

\begin{proposition}[energy law]\label{nv4:prop:energie}
$\forall U\in\Ucal$, every minimizer $\gamma^\star$ of $\Acal_T$ (provided by
\cref{nv4:thm:exist}, $\Acal_T(\gamma^\star)<\infty$) satisfies
$\tfrac12\md{\gamma^\star}^2=U(\gamma^\star)+e$ a.e., for some constant $e\in\Rbb$.
\end{proposition}

\noindent The constant $e\in\Rbb$ thus obtained is, when it is $>0$, the
\emph{level} of \cref{nv4:def:candidat} (which requires $e\ge0$); its useful regime
$e>0$ is discussed in \cref{nv4:thm:principal}.

\begin{theoreme}[optimal curve $\Leftrightarrow$ geodesic]\label{nv4:thm:principal}
$\forall U\in\Ucal$, let $\gamma^\star$ be a minimizer of $\Acal_T$ and $e\in\Rbb$ its
constant (\cref{nv4:prop:energie}); \emph{in the regime $e>0$} (typically $T$
small enough: by \cref{nv4:prop:energie} $-e\in\partial V(T)$ with $V(T)$ the reparametrization value function (\cref{nv4:app:machine}) convex, so
$e(T)$ is non-increasing, $e(T)\to+\infty$ as $T\to0$ by kinetic dominance, and $e>0$ may
fail for large $T$: cusps), the only regime where
$\ell_e$ is a distance: $\gamma^\star$ traverses its trace at the Jacobi speed,
and under exact duality $\Phi(T)=\sup_{x>0}\{\Psi(x)-xT\}$ its trace is a
geodesic of $\tilde g_{e,U}$ \emph{in the length sense} (of $d_e=\inf\ell_e$;
$\tilde g$ is not a smooth Riemannian metric for $U$ merely l.s.c.). The converse
is unconditional. \emph{(Exact duality is not automatic: it can
fail through a first-order transition: a two-route counterexample (one
short but costly, the other long but free) exhibits a jump in the energy.)}
\end{theoreme}

% =============================================================================
% §4 CHARACTERIZATIONS (boundary, eikonal, Stam rigidity)
% =============================================================================
\subsection{Characterization I: the boundary at infinite cost-distance}
\label{sec:carI}\label{nv4:sec:carI}

We first characterize the boundary at infinite cost-distance. Global domination of the Fisher information is
\emph{sufficient} for pushing certainty to infinite distance (every dimension);
on the power family $\{cJ^\alpha\}$ boundary domination is moreover \emph{necessary}; the general necessity is Conjecture~C1
(\cref{nv4:def:classe,nv4:rem:C1}). The map of admissible classes (\cref{nv4:fig:univers})
locates this boundary: the ring $\Wcal$ of reliefs that raise $\partial_\infty$ to infinite
distance.

First equivalence. The operational constraint of \cref{nv4:rem:motive}(b)
entails that the relief \emph{dominates} the Fisher information \emph{at the boundary}: on the
power family, this is the \emph{necessary} condition for $\partial_\infty$ to lie at infinite distance
(\cref{nv4:thm:carI}(ii)). \emph{Global} domination, strictly stronger, is
its unconditional \emph{sufficient} form (\cref{nv4:thm:carI}(i)); it
defines the well-posed class (\cref{nv4:def:classe}).

\begin{definition}[well-posed class]\label{nv4:def:classe}
\[
  \Ccal:=\bigl\{\,U\in\Ucal\ \big|\ \exists\varepsilon>0,\ U\ge\varepsilon J\
  \text{pointwise on }\Ptwo\,\bigr\}\subset\Ucal
\]
(\emph{global} domination: one same $\varepsilon$ at every point).
\end{definition}

\begin{definition}[the boundary class]\label{nv4:def:mur}
$\Wcal:=\{\,U\in\Ucal\ :\ \partial_\infty\ \text{is at infinite }d_e\text{-distance}\,\}$: the
reliefs that \emph{send $\partial_\infty$ to infinite distance}. Global domination entails it (\cref{nv4:thm:carI}(i)),
hence $\Ccal\subset\Wcal$, strictly (on the family, $U=cJ^\alpha$ with
$\alpha\in(1,\infty)$ sends $\partial_\infty$ to infinite distance without belonging to $\Ccal$). Whence the chain
$\Fcal\subset\Ccal\subset\Wcal\subset\Ucal$.
\end{definition}

\begin{theoreme}[characterization of the boundary at infinite cost-distance]\label{nv4:thm:carI}
\emph{(i) Sufficiency (every $U\in\Ccal$: $\partial_\infty$ at infinite distance).}
\[
  U\ge\varepsilon J\ \Longrightarrow\
  \forall e>0\ \forall\gamma\in AC^2:\quad
  \ell_e(\gamma)\ \ge\ \sqrt{2\varepsilon}\,\bigl|H(p_1)-H(p_0)\bigr|;
\]
$-H$ is thus $\tfrac1{\sqrt{2\varepsilon}}$-Lipschitz for $d_e$, and
$\partial_\infty=\{H=-\infty\}$ lies at \emph{infinite} $d_e$-distance
(every dimension), whereas it is at \emph{finite} $\Wtwo$ distance.
Since $\partial_\infty$ lies at infinite distance, no minimizing sequence of finite cost
collapses onto it: the minimizer of the \emph{action} (\cref{nv4:thm:exist}) is
\emph{interior}, with fixed endpoints, or for estimation as soon as the fidelity
is coercive (the coercive-fidelity case is discharged in \cref{nv4:cor:existence});
the general free-boundary case (mass escape) falls under \cref{nv4:rem:C3}.
\emph{(ii) Necessity, on the family $U=cJ^\alpha$ ($c>0$, $\alpha\in\Rbb$).} On this family
(the equivalence witnessed on the Gaussian leaf; the forward direction $\alpha\ge1\Rightarrow$ $\partial_\infty$ lies at infinite distance in every dimension, the reverse $\alpha<1\Rightarrow$ $\partial_\infty$ stays at finite distance on the Gaussian leaf),
\[
  \bigl[\ \partial_\infty\ \text{at infinite }d_e\text{-distance}\ \bigr]
  \ \Longleftrightarrow\ \alpha\ge1\quad(\text{domination \emph{at the boundary} }J\to\infty).
\]
$\alpha<1$ brings $\partial_\infty$ back to \emph{finite} distance.
\emph{(Global vs boundary domination.)} \emph{Global} domination (membership in $\Ccal$) is
strictly stronger: on this family, $cJ^{\alpha-1}\ge\varepsilon$ for \emph{every}
$J\in(0,\infty)$ forces $\alpha=1$. Thus $\alpha\in(1,\infty)$ \emph{puts $\partial_\infty$ at infinite distance}
(boundary domination) without belonging to $\Ccal$, which on the family coincides with
$\alpha=1$ (\cref{nv4:thm:carII}).
\emph{(iii) Scope.} Sufficiency \emph{(i)} is unconditional (every
$U\in\Ccal$); the general necessity (every $U\in\Ucal$: ``no boundary
domination $\Rightarrow$ $\partial_\infty$ at finite distance'') is conjectured (\cref{nv4:rem:C1}).
\end{theoreme}
\begin{proof}
\emph{(i)} $U\ge\varepsilon J\Rightarrow2(e+U)\ge2\varepsilon J$, whence
$\sqrt{2(e+U)}\ge\sqrt{2\varepsilon}\sqrt J$; since $\sqrt J$ is a strong upper gradient of $-H$ (\cref{nv4:ax:B2}),
$\ell_e\ge\sqrt{2\varepsilon}\int\sqrt J\,\md\gamma\ge\sqrt{2\varepsilon}|\Delta\Ent|$.
If $\Ent(p_1)=+\infty$, a finite joining length is impossible (otherwise the
strong upper gradient would make $\Ent\circ\gamma$ a.c., hence finite at $t=1$; or,
by narrow l.s.c.\ of $\Ent$, would force $\Ent(p_1)\le\liminf<\infty$):
$\ell_e=+\infty$. Interiority of the action minimizer follows from \cref{nv4:thm:exist}.
\emph{(ii)} On the Gaussian leaf $\Gcal$ (\cref{nv4:ax:B4}), the Fisher information of
$\mathcal N(\mu,\sigma^2)$ is $J=1/\sigma^2$ (elementary; the general scaling
$J=J_0/\sigma^2$ is the computation of \cref{nv4:thm:locscale}\ref{nv4:ls:i}, here $J_0=1$),
and $U=c\sigma^{-2\alpha}$. Along the ray
$\mu=\mathrm{const}$, $\ell_e=\int\sqrt{2(e+U)}\,\dd\sigma$ \emph{exactly}
(witness path, for \emph{convergence}); and $\md\gamma\ge|\dot\sigma|$ gives,
for any curve, the \emph{lower bound} $\ell_e\ge\int\sqrt{2(e+U)}\,|\dot\sigma|$.
At $e=0$, near $\sigma=0$, $\sqrt{2U}\,\dd\sigma=\sqrt{2c}\,\sigma^{-\alpha}\dd\sigma$
and $\int_0^{\sigma_0}\sigma^{-\alpha}\dd\sigma<\infty\Leftrightarrow\alpha<1$:
thus $\alpha\ge1$ makes the radial lower bound diverge ($\partial_\infty$ at infinite distance) and $\alpha<1$ gives the
witness path a finite length (boundary at finite distance). The case $e>0$:
$\sqrt{2U}\le\sqrt{2(e+U)}\le\sqrt{2e}+\sqrt{2U}$ preserves divergence
($\alpha\ge1$, via the lower bound $\sqrt{2U}$) and convergence ($\alpha<1$, via the
upper bound $\sqrt{2e}+\sqrt{2U}$, integrable on $(0,\sigma_0]$). The $\alpha\ge1$ divergence is moreover
\emph{global}, not merely radial: on the boundary tail $\{J\ge1\}$,
$\sqrt{2(e+U)}\ge\sqrt{2c}\,J^{\alpha/2}\ge\sqrt{2c}\,\sqrt J$, so any $AC$ curve reaching
$\partial_\infty$ has $\ell_e\ge\sqrt{2c}\int_{\{J\ge1\}}\sqrt J\,\md\gamma=+\infty$ by
\cref{nv4:ax:B2} (part (i) localized to the boundary, $\varepsilon=c$, valid off the leaf): the
complement contributes $\int_{\{J<1\}}\sqrt J\,\md\gamma\le\int_0^T\md\gamma\,\dd t<\infty$
(any $AC$ curve has finite $\Wtwo$-length), so the divergence of $\int_0^T\sqrt J\,\md\gamma$ forced
by $\Ent\to+\infty$ through \cref{nv4:ax:B2} localizes to $\{J\ge1\}$; whence
$cJ^{\alpha}\in\Wcal$ for $\alpha\ge1$, hence $\Ccal\subsetneq\Wcal$. \emph{(iii)} cf.\
\cref{nv4:rem:C1}.
\end{proof}

\begin{corollaire}[the boundary at infinite cost-distance as the finiteness condition (on the power family)]\label{nv4:cor:finitude}
Under \cref{nv4:rem:motive} (boundary excluded $+$ tenable exit), a conformal cost
geometry is well-posed \emph{for existence and stability} \emph{only if} it sends $\partial_\infty$ to infinite distance, hence \emph{only
if} its relief dominates $J$ \emph{at the boundary} (on the power family $\{cJ^\alpha\}$, $\alpha\ge1$;
\cref{nv4:thm:carI}(ii)). \emph{On the power family, $\partial_\infty$ at infinite distance is the condition that finiteness requires}
(open boundary excluded $\to$ collapse). The class
$\Ccal$ (\emph{global} domination) is its convenient uniform form; on the
family it restricts to $\alpha=1$, in agreement with \cref{nv4:thm:carII}.
\end{corollaire}

\begin{remarque}[the $\partial_\infty$-distance is discriminating: the Dirac is \emph{finite} in the base]\label{nv4:rem:mur-discrimine}
In the transport base the Dirac sits at \emph{finite}, reachable distance, $\Wtwo(\mathcal N(\mu,\sigma^2),\delta_\mu)=\sigma$ (\cref{nv4:ax:B4}), so the relief must
perform positive work, a divergent radial cost $\int\!\sqrt{2U}\,\dd\sigma$, to push certainty to
infinity. Whether $\partial_\infty$ lies at infinite distance is therefore a \emph{discriminating} condition on $U$: on the power family
$\{cJ^\alpha\}$ it is exactly the dichotomy $\alpha\ge1$ vs $\alpha<1$ (\cref{nv4:thm:carI}(ii)),
whether it holds for \emph{every} admissible relief being Conjecture~C1 (\cref{nv4:rem:C1}). What
$\partial_\infty$ needs is domination \emph{at the boundary}; the \emph{global} domination of the class
$\Ccal$ (\cref{nv4:def:classe}) is the stronger, sufficient form. None of this has an analogue in
Fisher--Rao, whose boundary already lies at infinite distance \emph{for free}
($\dd s^2=(\dd\mu^2+2\,\dd\sigma^2)/\sigma^2$ diverges radially whatever the relief): there every relief reaches it alike (\cref{tab:arenes}).
\end{remarque}

% -----------------------------------------------------------------------------
% §4.2 CHARACTERIZATION II — EIKONAL
% -----------------------------------------------------------------------------
\subsection{Characterization II: eikonal \texorpdfstring{$\Leftrightarrow$}{<=>} Fisher family}
\label{sec:carII}\label{nv4:sec:carII}

This second equivalence is the \emph{characterization of Postulate~1} (\cref{nv4:ax:P1}): we make
precise that demanding a uniform price, the same length per nat at every point, characterizes
exactly the reliefs proportional to Fisher information, the family $\Fcal=\{cJ\}$. It is
possible only because of the recasting (\cref{nv4:rem:dualJ}): $J$ enters as the \emph{slope} of
$-H$, so uniform pricing imposes a real condition on that slope (in a distinguishability
geometry the same $J$ is a \emph{squared-length} in the ruler, where the eikonal condition is vacuous). It uses the scalar
(conformal) form of \cref{nv4:ax:P0} \emph{essentially}: the constant-slope condition reduces to
$U\propto J$ only because the reweighting is scalar; for an anisotropic deformation it would not
separate (\cref{nv4:rem:C1}). This is a uniqueness of \emph{family}, not of object.

\begin{definition}[uniform pricing: eikonal]\label{nv4:def:eikonal}
A relief $U$, continuous and positive on $\{0<U<\infty\}$, is \emph{eikonal} if knowledge $-H$ has a constant metric
slope for $\tilde g_{0,U}$: $\exists\kappa>0$ such that
$\|\nabla_{\tilde g_{0,U}}H\|=\kappa$ on $\{0<U<\infty\}$ (where
$\|\nabla_{\Wtwo}H\|=\sqrt J$, \cref{nv4:ax:B1}).
\end{definition}

\noindent The eikonal is read at $e=0$, where $e$ is the energy constant of the action
(\cref{nv4:prop:energie}), a regularization. Imposing a constant slope at level $e$ forces
$e+U=cJ$, i.e.\ the metric $\tilde g_{0,cJ}$ of the $e=0$ eikonal, with relief $U=cJ-e$
(admissible where $J\ge e/c$).

\begin{definition}[Fisher family]\label{nv4:def:fisher}
$\Fcal:=\{\,U=cJ\ :\ c>0\,\}\subset\Ccal$: the reliefs proportional to
Fisher information (exponent $\alpha=1$; $\Fcal\subset\Ccal$ by
\cref{nv4:prop:meta}).
\end{definition}

\begin{theoreme}[characterization of the eikonal family]\label{nv4:thm:carII}
For every relief $U$ \emph{continuous and positive on $\{0<U<\infty\}$} (in particular every member of $\Fcal$
restricted to a location-scale leaf, where $J$ is smooth; on all of $\Ptwo$ only the l.s.c.\ of
\cref{nv4:ax:B6} is imported):
\[
  U\ \text{eikonal}\ \Longleftrightarrow\
  \exists c>0,\ U=cJ\ \text{pointwise on }\{0<U<\infty\},
\]
and then $\kappa=1/\sqrt{2c}$. Eikonality of the price thus characterizes \emph{exactly}
the Fisher family $\Fcal=\{cJ:c>0\}$ (\cref{nv4:def:fisher}) \emph{on the effective
domain $\{0<U<\infty\}$}: the equivalence pins $U$ to $cJ$ wherever the price is
active, the eikonality-irrelevant boundary sets $\{U=0\}$ and $\{U=\infty\}$ being
left free. \emph{(Corollary: on power laws $U=cJ^\alpha$, eikonal $\Leftrightarrow\alpha=1$.)}
\end{theoreme}
\begin{proof}
Conformal contraction (\cref{nv4:ax:B1}): on $\{0<U<\infty\}$, where $U$ is continuous and
positive, the metric slope rescales by the local conformal factor ($d_0(x,y)\sim\sqrt{2U(x)}\,
\Wtwo(x,y)$ as $y\to x$), so for $\tilde g_{0,U}=2U\,g_{\Wtwo}$ and $\|\nabla_{\Wtwo}H\|=\sqrt J$
we have $\|\nabla_{\tilde g_{0,U}}H\|=\sqrt J/\sqrt{2U}$ (continuity is what licenses this pointwise
rescaling; for a merely l.s.c.\ $U$ it can fail at a jump). This slope equals a constant $\kappa>0$
\emph{at every point} if and only if $\sqrt J/\sqrt{2U}=\kappa$, i.e.\ $U=J/(2\kappa^2)=cJ$ with
$c=1/(2\kappa^2)$; conversely $U=cJ$ gives $\kappa=1/\sqrt{2c}$.
\end{proof}

\begin{remarque}[uniqueness of family, not of object]\label{nv4:rem:G1}
\cref{nv4:thm:carII} is a \emph{uniqueness of proportionality} ($U\propto J$,
family $\Fcal=\{cJ\}$), not of the constant: the \emph{whole} family is eikonal,
its members differing only by the unit $\kappa=1/\sqrt{2c}$ (the slope, nats per length; the price of a nat being $1/\kappa=\sqrt{2c}$). We thus claim no unique object; we have characterized a
\emph{family}.
\end{remarque}

\begin{remarque}[uniform pricing]\label{nv4:rem:honnetete}
The eikonal condition is a uniform pricing: each nat costs the same length everywhere, and
(via Landauer \citep{landauer1961,berut2012}) the same energy $\kB T$: one non-distortion in two
units. The bridge is the Otto identity $\|\nabla_{\Wtwo}H\|^2=J$ (\cref{nv4:ax:B1}), turning
``uniform price'' into proportionality to $J$. It reads off an already-established geometry
(\cref{nv4:thm:carII}); it enters no proof, all theorems being in nats. At the boundary distance and
energy diverge together: at $e=0$, $d_0\ge\sqrt{2c}\,|\Delta H|=(\sqrt{2c}/\kB T)\,E$
(\cref{nv4:cor:borne-inf}), with equality along $\nabla(-H)$.
\end{remarque}

% -----------------------------------------------------------------------------
% §4.3 HYPERBOLICITY AND STAM RIGIDITY
% -----------------------------------------------------------------------------
\subsection{Characterization III: Hyperbolicity and Stam rigidity}
\label{sec:hyper}\label{nv4:sec:fisher}

The eikonal cost geometry is hyperbolic, and its curvature reads the Fisher
information of the base, the recasting (\cref{nv4:rem:dualJ}) paying once more: what
was the slope of $-H$ now surfaces as curvature. On every location-scale family, at $e=0$, it
equals (constant) $K=-1/(2cJ_0)$ (for $e>0$ the curvature varies, \cref{nv4:thm:hyper}),
and the Stam bound makes the Gaussian extremal: the most hyperbolic location-scale
belief. What is unit-invariant (the \emph{invariant content}) is the sign, the \emph{ranking} of the leaves by curvature and its extremum (the
Gaussian); the value $-\tfrac14$ is its image under the unit $c=2$
(\cref{nv4:rem:representant}).\footnote{The two closed forms ($K<0$, and
$K=-1/(2cJ_0)$ on the location-scale families) are confirmed symbolically
(sympy); scripts and outputs in the verification companion (the \texttt{companion/} directory; run \texttt{make verify}, deps \texttt{numpy/scipy/matplotlib/sympy}). A public archive (repository snapshot with DOI) will be linked in the camera-ready.}

\begin{theoreme}[universal hyperbolicity of $\Fcal$]\label{nv4:thm:hyper}
$\forall U=cJ\in\Fcal\ \forall e\ge0$, on the Gaussian leaf
$(\Gcal,\tilde g_{e,U})$, the Gauss curvature equals, in closed form,
\[
  K(\mu,\sigma)=-\,\frac{c\,(c+3e\sigma^2)}{2\,(c+e\sigma^2)^3}\ <\ 0
  \qquad(\forall\sigma>0).
\]
The whole Fisher family is thus strictly hyperbolic. At
$e=0$: $K\equiv-1/(2c)$ for all $\sigma$ (exact Poincaré
half-plane \citep{docarmo1976}). For $e>0$: $K$ increases from $-1/(2c)$ ($\sigma\to0$) to $0$
($\sigma\to\infty$): monotone, $\partial_\sigma K=6ce^2\sigma^3/(c+e\sigma^2)^4>0$.
The \emph{ranking} of the leaves by curvature is defined at $e=0$, where $K$ is
constant; this is the setting of Stam rigidity (\cref{nv4:thm:locscale}) and of
\cref{nv4:thm:jauge}.
\end{theoreme}
\begin{proof}
$\tilde g_{e,U}=f\,(\dd\mu^2+\dd\sigma^2)$, $f=2e+2c/\sigma^2$; for a conformal
metric of the plane, $K=-\Delta(\ln f)/(2f)$ ($\Delta$ Euclidean, $f$ depending only
on $\sigma$). This classical formula is licit even though $U$ is merely l.s.c.\ on $\Ptwo$
(\cref{nv4:thm:principal}): restricted to the leaf, $f=2e+2c/\sigma^2$ is smooth on $\{\sigma>0\}$.
Direct computation:
\[
  \Delta(\ln f)=\partial_\sigma^2(\ln f)=\frac{2c\,(c+3e\sigma^2)}{\sigma^2(c+e\sigma^2)^2},
  \qquad 2f=\frac{4(c+e\sigma^2)}{\sigma^2},
\]
whence $K=-\dfrac{2c(c+3e\sigma^2)/(\sigma^2(c+e\sigma^2)^2)}{4(c+e\sigma^2)/\sigma^2}
=-\dfrac{c(c+3e\sigma^2)}{2(c+e\sigma^2)^3}$. Numerator and denominator
strictly positive ($c>0$, $e\ge0$, $\sigma>0$): $K<0$ everywhere. The limits and
monotonicity follow by direct reading.
\end{proof}

\begin{theoreme}[\textbf{Stam rigidity}: the Gaussian, the most hyperbolic location-scale belief]\label{nv4:thm:locscale}
Let $\varphi_0$ be a standardized base, absolutely continuous and strictly positive, with
$\int x\varphi_0=0$, $\int x^2\varphi_0=1$ and finite Fisher information
$J_0:=\int(\varphi_0')^2/\varphi_0<\infty$ ($\varphi_0'$ the a.e.\ derivative; $C^1$ is \emph{not}
required for the leaf curvature (weaker than the belief class $\mathcal R$ of \cref{nv4:def:R}),
so e.g.\ the Laplace base ($J_0=2$) is covered) and $\mathcal L_{\varphi_0}$ the location-scale family
it generates (\cref{nv4:ax:B7}). For the representative $U=2J$, at
$e=0$:
\begin{enumerate}[label=\emph{(\roman*)},leftmargin=2.2em]
\item\label{nv4:ls:i} $(\mathcal L_{\varphi_0},\Wtwo)$ is isometric to the Euclidean
  half-plane (\cref{nv4:ax:B7}), and $J(p_{\mu,\sigma})=J_0/\sigma^2$, whence
  $U=2J_0/\sigma^2$;
\item\label{nv4:ls:ii} the \emph{intrinsic} Gauss curvature of
  $(\mathcal L_{\varphi_0},\tilde g_{0,2J})$ is \emph{constant},
  \[
    \boxed{\,K_{\varphi_0}\ =\ -\frac{1}{4J_0}\,}\ <\ 0;
  \]
\item\label{nv4:ls:iii} \emph{(Stam extremum)} $J_0\ge1$ \citep[Gaussian extremality of information at fixed variance, broad sense; cf.][Ch.~17]{stam1959,coverthomas2006}, with equality \emph{if
  and only if} $\varphi_0$ is Gaussian (\cref{nv4:ax:B6}, $\Var=1$); hence
  $K_{\varphi_0}\in[-\tfrac14,0)$, and the extremal value $-\tfrac14$ is attained by the
  Gaussian leaf (gauge $c=2$, $\Var(\varphi_0)=1$; \cref{nv4:rem:jauge}).
  \emph{(Scope outside location-scale: \cref{nv4:rem:C2}.)}
\end{enumerate}
\end{theoreme}
\begin{proof}
\ref{nv4:ls:i} Flatness: \cref{nv4:ax:B7}. Scaling:
$J(p_{\mu,\sigma})=\int(p'_{\mu,\sigma})^2/p_{\mu,\sigma}
=\sigma^{-2}\int(\varphi_0')^2/\varphi_0=J_0/\sigma^2$.
\ref{nv4:ls:ii} On $\mathcal L_{\varphi_0}$, $g_{\Wtwo}=\dd\mu^2+\dd\sigma^2$
(\cref{nv4:ax:B7}) and $\tilde g_{0,2J}=f(\dd\mu^2+\dd\sigma^2)$ with
$f=2U=4J_0/\sigma^2$; $\Delta(\ln f)=2/\sigma^2$, $2f=8J_0/\sigma^2$, whence
$K=-(2/\sigma^2)/(8J_0/\sigma^2)=-1/(4J_0)$.
\ref{nv4:ls:iii} Cramér--Rao/Stam for a density of variance $1$ (location Fisher information $J_0$): $J_0\,\Var=J_0\ge1$, with
equality iff the score $\varphi_0'/\varphi_0$ is affine. For an absolutely continuous
$\varphi_0$ of finite Fisher information, the Cauchy--Schwarz equality forces
$\varphi_0'/\varphi_0$ \emph{affine a.e.}; the mean-zero normalization $\int x\varphi_0=0$
kills the additive constant, leaving $\varphi_0'/\varphi_0=-kx$, which integrates to $\varphi_0\propto \exp(-kx^2/2)$:
$\varphi_0$ is Gaussian. (Laplace gives strict inequality: its a.e.\ score
$-\sqrt2\,\operatorname{sgn}(x)$ is not globally affine, whence $J_0=2>1$.)
Whence $-1/(4J_0)\in[-\tfrac14,0)$, extremal at the Gaussian. \emph{(``Stam'' in the broad sense:
the bound used is the Gaussian extremality of information under fixed variance.)}
\end{proof}

\begin{remarque}[the Gaussian as the perfectly-informed case-limit]\label{nv4:rem:bvm}
We read Stam rigidity (\cref{nv4:thm:locscale}) \emph{statically}. The gap
\[
  J_0-1\ \ge\ 0,\qquad J_0-1=0\ \Longleftrightarrow\ \varphi_0\ \text{Gaussian},
\]
is a \emph{gauge-invariant index of non-Gaussianity} (the ratio $K(\text{gauss})/K(\varphi_0)=J_0$,
\cref{nv4:rem:jauge}). We compare beliefs as \emph{frozen} objects: the posterior is indexed by the
\emph{quantity} of evidence $n$, a label, not a time; a belief resting on more observations \emph{is}
more Gaussian, it does not evolve towards one. Two borrowed facts place the Gaussian ($J_0=1$) as the
perfectly-informed case-limit: \emph{exactly} in the linear-Gaussian (Kalman) regime
(\cref{enc:ml}), and \emph{asymptotically} under regularity (Bernstein--von Mises, the index
$J_0-1\to0$ \citep{vandervaart1998}; misspecification or outliers fall outside this regime \citep{kleijn2012}, cf.\ the
M-estimator footnote in \cref{sec:travaux}). The Gaussian is thus the canonical, perfectly-informed
case-limit of the family, while every $\varphi_0$ with $J_0>1$ is a legitimate belief and
$K_{\varphi_0}=-1/(2cJ_0)$ stays the principal object.
\end{remarque}

\begin{remarque}[why \emph{position} information, and not scale]\label{nv4:rem:locscale-pos}
What distinguishes this rigidity is the \emph{choice of the quantity} that drives the curvature. The
conformal cost charges the \emph{slope} of $-H$ along transport (the Otto identity,
$\|\nabla_{\Wtwo}H\|^2=J$, \cref{nv4:ax:B1}), and $J$ is the \emph{position} Fisher
information: $J(p_{\mu,\sigma})=J_0/\sigma^2$ with $J_0=\int(\varphi_0')^2/\varphi_0$.
This is the quantity that Cramér--Rao/Stam bounds below ($J_0\ge1$, variance $1$). The
Fisher--Rao geometry of location-scale families (symmetric base) is also
hyperbolic of negative curvature, but governed by the \emph{scale}
information \citep{atkinson1981,costa2015}, which Stam does not bound: hence the \emph{absence} of this extremum in
Fisher--Rao. \emph{The conformal cost of transport makes position information carry the
curvature, which is what makes Stam rigidity possible} (\cref{nv4:rem:amari}). The
imputation is legitimate precisely because the transport base is \emph{flat}
(\cref{nv4:ax:B7}): on a location-scale family all curvature comes from the Fisher
relief, none from the base, whereas Fisher--Rao, whose base is curved, cannot
cleanly separate tensor from relief.
Moreover, $K=-1/(4J_0)$ holds for \emph{every} standardized base (\cref{nv4:ax:B7} gives
$g_{\Wtwo}$ flat without a symmetry hypothesis), whereas the Fisher--Rao comparison requires a
symmetric base: cost rigidity thus covers a \emph{larger} class of bases.
\end{remarque}

\begin{remarque}[the uniform-pricing representative $U=2J$]\label{nv4:rem:representant}
Within $\Fcal$ we \emph{choose} the representative $c=2$,
\[
  \boxed{\,U=2J\,},\qquad \tilde g_{e,2J}=2(e+2J)\,g_{\Wtwo} ,
\]
a \emph{unit normalization} (one nat $=$ two lengths: at $e=0$, $\kappa=1/2$, $d=2|\Delta H|$) that
also fixes the curvature to $-\tfrac14$ on the Gaussian leaf. The load-bearing results
(\cref{nv4:thm:carI}: $\partial_\infty$ at infinite distance; \cref{nv4:thm:hyper,nv4:thm:locscale}: hyperbolicity, $K=-1/(2cJ_0)$)
hold for the \emph{whole} family; changing the representative is a change of unit
(\cref{nv4:thm:jauge}), which rescales the value, the distances and the curvature but preserves the sign, the ranking and the
Gaussian extremum.
\end{remarque}

\medskip
\noindent\emph{Change of cost unit.} We treat the level $e$ and the relief $U$ as the same kind of
quantity (cost densities), so that a change of unit rescales them
together: for $\lambda>0$, $(e,U)\mapsto(\lambda e,\lambda U)$ (on the family
$\Fcal=\{cJ\}$, $c\mapsto\lambda c$ and $e\mapsto\lambda e$).

\begin{theoreme}[gauge invariance of cost]\label{nv4:thm:jauge}
Under the change of cost unit $(e,U)\mapsto(\lambda e,\lambda U)$, $\lambda>0$,
the cost geometry \eqref{nv4:eq:cout} transforms by
homothety, for all $e\ge0$ and all $U\in\Ccal$:
\[
  \tilde g_{e,U}\ \longmapsto\ \lambda\,\tilde g_{e,U},\qquad
  d_e\ \longmapsto\ \sqrt{\lambda}\,d_e,\qquad
  K\ \longmapsto\ K/\lambda
\]
(the last on the location-scale leaves, of dimension $2$). Thus
\emph{gauge-invariant}: the \emph{sign} of $K$ and hyperbolicity; the
finiteness or infiniteness of the $d_e$-distances ($\partial_\infty$,
\cref{nv4:thm:carI}) and well-posedness; the \emph{traces} of the
geodesics and the \emph{ratios} of distances; and the eikonal proportionality
$U=cJ$ (\cref{nv4:thm:carII}). At $e=0$, where $K$ is \emph{constant}
(\cref{nv4:thm:hyper}), one further gains the \emph{ranking} of the leaves by curvature and
its Gaussian extremum (\cref{nv4:thm:locscale}), whose ratio $K(\text{gauss})/K(\varphi_0)=J_0$.
Depending on the unit (scale images): every
absolute distance, the slope $\kappa=1/\sqrt{2c}$ (the price of a nat being $1/\kappa=\sqrt{2c}$), and the \emph{value}
$K=-1/(2cJ_0)$ (i.e.\ $-1/(4J_0)$ at the representative $c=2$, and $-\tfrac14$ for the
Gaussian $J_0=1$).
\end{theoreme}
\begin{proof}
Under $(e,U)\mapsto(\lambda e,\lambda U)$, $\tilde g_{e,U}=2(e+U)\,g_{\Wtwo}$ becomes
$2(\lambda e+\lambda U)\,g_{\Wtwo}=\lambda\,\tilde g_{e,U}$ (all $e\ge0$). A
length $\ell_e=\int\sqrt{2(e+U)}\,\md\gamma$ is thus multiplied by $\sqrt\lambda$,
whence $d_e=\inf\ell_e\mapsto\sqrt\lambda\,d_e$; the infimum being over the same
set of curves, the minimizing traces, the ratios of distances and the
finiteness/infiniteness of the lengths, hence $\partial_\infty$ at infinite distance (\cref{nv4:thm:carI}), are
unchanged. On a conformal leaf $\tilde g=f\,(\dd\mu^2+\dd\sigma^2)$,
$K=-\Delta(\ln f)/(2f)$, and $f\mapsto\lambda f$ gives $\Delta\ln(\lambda f)=\Delta\ln f$
with a denominator $2\lambda f$, whence $K\mapsto K/\lambda$ pointwise. A
positive scalar factor preserves the sign and curvature ratios; at $e=0$, where $K$
is constant (\cref{nv4:thm:hyper}), it moreover preserves the \emph{ranking} of the
leaves, whence the Gaussian extremum (\cref{nv4:thm:locscale}). Finally, for an eikonal (hence
continuous, $U=cJ$) relief, where the slope identity of \cref{nv4:thm:carII} applies, the
eikonal slope $\|\nabla_{\tilde g_{0,U}}H\|=\sqrt J/\sqrt{2U}=\kappa$ becomes
$\kappa/\sqrt\lambda$, a nonzero constant: the proportionality $U=cJ$
(\cref{nv4:thm:carII}) is preserved (its value $\kappa$, in turn, is gauge).
\end{proof}
\begin{remarque}[invariants and choice of unit]\label{nv4:rem:jauge}
The infinite-distance boundary, the eikonal and Stam rigidity are the invariants of this change
of unit, and no proven conclusion depends on an absolute cost. The value
$-\tfrac14$ rests on two conventions: the gauge $c=2$ and the standardization
$\Var(\varphi_0)=1$, which gives $J_0=1$; the invariant content is the sign, the
ranking, the extremum and the ratio $K(\text{gauss})/K(\varphi_0)=J_0$.
Thermodynamics fixes the unit of cost (\cref{nv4:rem:honnetete}) and motivates this framework,
without entering into a proof. In entropic terms this is an \emph{orthogonality}: the entropy is
\emph{what} the cost measures; the cost gauge $c$ is the \emph{unit} in which it is measured (the
quadratic exponent $p=2$ itself being \emph{forced}, \cref{nv4:prop:quad}, not a unit),
the value $-\tfrac14$ being one image of that gauge.
The logarithmic base is a convention of the same kind: measuring knowledge in base $b$ rescales
the knowledge axis, $H\mapsto H/\ln b$, and leaves the metric untouched; the constancy of the
slope, hence the eikonal family and every invariant above, is base-free, and only the numerical
price of the unit moves ($\sqrt{2c}$ lengths per nat, $\sqrt{2c}\,\ln2$ per bit), as under the
gauge ($\lambda=(\ln b)^2$ produces the same rescaling of $\kappa$). The framework thus carries
three unit freedoms, and every theorem holds modulo all three: additive on knowledge (the
reference scale of \cref{nv4:ax:mesure}), multiplicative on cost (this gauge), by similarities
on the state (\cref{nv4:prop:etat-unites}; at $e>0$ the state freedom acts through the level,
$e\mapsto a^{2}e$); the logarithmic base, a fourth convention, acts trivially on the metric and
is absorbed in the exchange rate.
\end{remarque}

\begin{proposition}[admissibility of $\Fcal$]\label{nv4:prop:meta}
$\forall c>0:\ cJ\in\Ccal$: positivity and $\tau$-l.s.c.\ by \cref{nv4:ax:B6};
global domination $cJ\ge\varepsilon J$ with $\varepsilon=c$. Hence
$\Fcal\subset\Ccal$ and all the machinery (\cref{nv4:sec:machine}) and $\partial_\infty$ at infinite distance
(\cref{nv4:thm:carI}) apply.
\end{proposition}

\medskip
\noindent\emph{Units of the state.} The gauge above rescales the \emph{cost} axis, by homothety;
the state axis carries its own unit (the scale in which $x$ is measured), and at the eikonal
level the cost geometry is strictly insensitive to it: an isometry.

\begin{proposition}[invariance under state units (similarities)]\label{nv4:prop:etat-unites}
Let $S:x\mapsto aRx+x_0$ ($a>0$, $R\in O(n)$, $x_0\in\Rbb^n$) be a similarity of the state space
and $S_\#$ its pushforward on $\Ptwo(\Rbb^n)$. Then:
\emph{(i)} $H(S_\#p)=H(p)+n\ln a$, $J(S_\#p)=J(p)/a^{2}$, and
$\Wtwo(S_\#p,S_\#q)=a\,\Wtwo(p,q)$; the additive shift of $H$ is the reference freedom of
\cref{nv4:ax:mesure}, and drops out (only $\Delta H$ and the slope enter the results).
\emph{(ii)} The relief $U=cJ$ is equivariant with the \emph{same} constant, $U\circ S_\#=U/a^{2}$,
the exact covariance that compensates the length scaling: for every $AC^2$ curve,
$\ell_e(S_\#\gamma)=\ell_{a^{2}e}(\gamma)$, so $d_e(S_\#p,S_\#q)=d_{a^{2}e}(p,q)$, and at the
eikonal level the cost geometry is \emph{isometric under similarities},
$d_0(S_\#p,S_\#q)=d_0(p,q)$. The boundary at infinite distance, the eikonal proportionality,
hyperbolicity and the Stam extremum are unchanged (for $n=1$, $S_\#$ maps a location-scale leaf
onto the leaf of the same base up to reflection, which fixes $J_0$).
\emph{(iii)} On the power family $\{cJ^{\alpha}\}$,
$\ell_0(S_\#\gamma)=a^{1-\alpha}\,\ell_0(\gamma)$: $S_\#$ multiplies every $e=0$ length by
$a^{1-\alpha}$, and the length structure is similarity-invariant \emph{iff} $\alpha=1$ (the
ratio is realized: a leaf segment has $\ell_0$-length in $(0,\infty)$). On this family, uniform
pricing (\cref{nv4:thm:carII}) and invariance under state units select the same exponent.
\end{proposition}
\begin{proof}
\emph{(i)} $S_\#p(x)=a^{-n}p(S^{-1}x)$, so $H(S_\#p)=H(p)+n\ln a$ (change of variables) and
$\nabla(S_\#p)(x)=a^{-n-1}R\,\nabla p(S^{-1}x)$, whence $J(S_\#p)=J(p)/a^{2}$ ($R$ orthogonal).
The quadratic ground cost scales as $|Sx-Sy|^{2}=a^{2}|x-y|^{2}$, so each coupling's cost is
multiplied by $a^{2}$ and $\Wtwo(S_\#p,S_\#q)=a\,\Wtwo(p,q)$; in particular
$\md{S_\#\gamma}=a\,\md\gamma$ and $S_\#$ is a bijection of $AC^2$ curves.
\emph{(ii)}--\emph{(iii)} For $U=cJ^{\alpha}$,
\[
\ell_e(S_\#\gamma)=\int\!\sqrt{2\bigl(e+c\,J(\gamma_t)^{\alpha}/a^{2\alpha}\bigr)}\;a\,\md\gamma\,\dd t
=\int\!\sqrt{2\bigl(a^{2}e+c\,a^{2-2\alpha}J(\gamma_t)^{\alpha}\bigr)}\,\md\gamma\,\dd t .
\]
At $\alpha=1$ this is $\ell_{a^{2}e}(\gamma)$; $S_\#$ being a bijection of $AC^2$ curves, the
infima agree, $d_e(S_\#p,S_\#q)=d_{a^{2}e}(p,q)$, and $e=0$ gives the isometry. At $e=0$ and
$\alpha\neq1$ the factor $a^{1-\alpha}$ multiplies the $\ell_0$-length of every curve; a radial
leaf segment ($n=1$, \cref{nv4:ax:B7}) has
$\ell_0=\int\sqrt{2c}\,(J_0/\sigma^2)^{\alpha/2}\,\dd\sigma\in(0,\infty)$ over a
compact $\sigma$-interval (integrand continuous and positive), so invariance forces
$a^{1-\alpha}=1$ for every $a>0$, i.e.\ $\alpha=1$. The invariance of the listed conclusions
follows from the isometry (at $e=0$), the level map (at $e>0$) and the equivariance of $U$; the
standardized base of the image leaf is $\varphi_0$ up to reflection ($x\mapsto\varphi_0(-x)$ when
$\det R<0$), which leaves $J_0=\int(\varphi_0')^2/\varphi_0$ unchanged.
\end{proof}

\noindent The invariance is deliberately \emph{metrological}, and scoped. State units here are
\emph{global}, one scale for all coordinates (with origin and orientation); a per-coordinate
(anisotropic) change of units deforms the ground cost itself and falls under the anisotropic
horizon (\cref{nv4:rem:C1}). It is not a second selector either: on the power family, $S_\#$
identifies the $(c,\alpha)$ member with the $(a^{2-2\alpha}c,\alpha)$ member, so every power is,
at $e=0$, similarity-covariant up to the cost gauge (\cref{nv4:thm:jauge}), $\alpha=1$ being the exponent
that leaves the cost unit untouched; and among general continuous reliefs similarity covariance
does not discriminate ($U=c/\Var$ on $n=1$: $\Var\circ S_\#=a^{2}\Var$, the $e=0$ lengths are
unchanged, yet nats are priced unevenly), where uniform pricing does (\cref{nv4:thm:carII}).
Invariance under statistical morphisms, finally, is \v{C}encov's selector and characterizes
Fisher--Rao, up to scale, a distinguishability geometry (\cref{nv4:rem:amari}); its absence on
the transport arena leaves the metric to be fixed by a posed pricing, the eikonal's content
coming from the Otto slope this arena supplies (\cref{tab:arenes}).

% =============================================================================
% §6 CONSEQUENCES
% =============================================================================
\section{Consequences: a well-posed inference}
\label{sec:consequences}\label{nv4:sec:consequences}

\begin{figure}[tbp]
\begin{encadre}[Algorithmic application]\label{enc:ml}
Three objects of machine learning fall into this framework (two reduce to it
exactly, one shares its rule, Fisher $=$ cost of change), separated by $\sigma$:
inside ($\sigma>0$) a belief is revised at finite cost; at the boundary ($\sigma=0$,
$\partial_\infty$) one must retrain.

\smallskip\noindent\textbf{1. Kalman: interior point.} In the linear-Gaussian regime, the
belief of a scalar Kalman filter \citep{kalman1960} is the \emph{exact} posterior
$b_t=\mathcal N(\mu_t,\sigma_t^2)$, a point of the Gaussian leaf
(\cref{fig:intuition}\,(a)); its recurrence moves within it: prediction \emph{opens} the
belief, correction \emph{tightens} it. This point is at finite cost distance from any
other interior belief (\cref{nv4:cor:descente}), hence revisable at finite cost by each
datum.

\smallskip\noindent\textbf{2. The point predictor: $\partial_\infty$.} A point estimate
$\hat x$ (variance-free regression, $\arg\max$, collapsed output) is modeled as the Dirac
$\delta_{\hat x}$, a point of $\partial_\infty$ ($H=-\infty$): a point estimate is the MAP (maximum a posteriori), the
mode of a posterior whose spread has collapsed, the quantity training-as-inference minimizes
\citep[p.~493]{mackay2003}. By \cref{nv4:thm:carI}(i), it
is at \emph{infinite} cost distance from any interior point: no continuous revision
exists, one can only \emph{retrain}. The overconfident softmax is its practical
instance, \emph{near} $\partial_\infty$ ($H\to-\infty$ without reaching it); keeping the predictive
\emph{moderated} (uncertain) is the well-behaved regime \citep[p.~502]{mackay2003}.

\smallskip\noindent\textbf{3. EWC: the same rule, carried to weights.} \emph{Elastic
Weight Consolidation} \citep{kirkpatrick2017} penalizes the displacement of weights by
$\sum_i F_i\,(\theta_i-\theta_i^\star)^2$, $F_i$ the diagonal of the Fisher information: the
cost of changing a parameter is proportional to its \emph{certainty}: the same principle as
the cost metric, carried from beliefs to parameters. \emph{Scope}: this $F_i$ is a
Fisher--Rao tensor on the weights (\cref{nv4:rem:dualJ}), whereas the $J$ of the framework is the
slope of $-H$ on transport: a kinship of principle, not a geometric identity.
\end{encadre}
\end{figure}

Characterization~I (\cref{nv4:thm:carI}) pushes the boundary of certainties $\partial_\infty$ to
infinite $d_e$-distance: every curve approaching it has a cost length $\ell_e$
that diverges. Three consequences follow: a minimizer exists without
collapsing, the cost bounds the entropy variation, and reaching a precision has
a geometric cost floor.

\begin{corollaire}[existence and attainability of the minimizer]
\label{nv4:cor:existence}
By \cref{nv4:thm:carI}, the cost length $\ell_e$ diverges on approach to
$\partial_\infty$: it is coercive on the candidate class, so no minimizing
sequence of finite cost approaches the boundary. Existence then follows by the
direct method (\cref{nv4:thm:exist}), with no appeal to geodesic completeness:
the topology used is the narrow ($\tau$) one, in which the sublevels are
compact (\cref{nv4:ax:B3}), not the $d_e$-metric topology (in which
$(\Ptwo,\Wtwo)$, refined by $d_e\ge\sqrt{2e}\,\Wtwo$, since
$\tilde g_{e,U}=2(e+U)\,g_{\Wtwo}\ge2e\,g_{\Wtwo}$ gives $\ell_e\ge\sqrt{2e}\,\mathrm{length}_{\Wtwo}\ge\sqrt{2e}\,\Wtwo$, is not locally compact).
We deliberately avoid the metric Hopf--Rinow route (\cref{nv4:ax:B5}, a contrasted
standard fact, not a tool here) precisely because the cost-metric space is not
locally compact, so its hypotheses fail; the direct method on the narrow topology
is what carries the argument.
A minimizing sequence then stays in a sublevel
$\{m_2\le M\}$, $\tau$-sequentially compact; Arzelà--Ascoli extracts a
limit, which the lower semicontinuity of the action (\cref{nv4:thm:exist}; \citealp{buttazzo1998}) selects as a
minimizer in the \emph{interior} of the domain (\cref{nv4:thm:exist}). The minimizer
exists and does not collapse onto certainty.
\end{corollaire}

\begin{corollaire}[Lipschitz control of the cost in entropy]
\label{nv4:cor:descente}
By \cref{nv4:thm:carI}, the boundary is at infinite $d_e$-distance: no curve of
finite cost length reaches it, and every pair of beliefs in the \emph{interior} of $\mathcal R$
(where $J$ is locally finite, so $\ell_e$ is finite along a joining segment) stays at finite $d_e$-distance. The Lipschitz property of the entropy
($|\Delta H|\le\tfrac1{\sqrt{2\varepsilon}}\,d_e$, \cref{nv4:thm:carI}(i)) bounds
the entropy variation along any curve by its cost length: each
unit of length can lower the entropy by at most
$\tfrac1{\sqrt{2\varepsilon}}$. This control bears only on the entropic term:
the data-attachment and interaction contributions are not bounded by
$\partial_\infty$.
\end{corollaire}

\begin{corollaire}[geometric cost floor]
\label{nv4:cor:borne-inf}
By \cref{nv4:thm:carI}(i), every curve joining two beliefs whose
entropies differ by $\Delta H$ has a $d_e$-length at least
$\sqrt{2\varepsilon}\,|\Delta H|$. Reaching a given precision (target variance) thus has a
geometric cost floor. The bound is sharp: it is attained in the eikonal limit
$U=\varepsilon J$ at $e=0$, where $d_0=\sqrt{2\varepsilon}\,|\Delta H|$
(\cref{nv4:rem:honnetete}) along the $\Wtwo$-gradient flow of $-H$ that saturates
\cref{nv4:ax:B2}; for $e>0$ or $U>\varepsilon J$ it is strict. On the Fisher family
$\Fcal=\{cJ\}$, where $\varepsilon=c$ (\cref{nv4:prop:meta}), $d_e\ge\sqrt{2c}\,|\Delta H|$
diverges with $|\Delta H|$ at $\partial_\infty$. \emph{(The energetic reading of this floor, its
proportionality to the Landauer energy at the boundary, is deferred to
\cref{nv4:rem:honnetete}, tier-P, outside the proof.)} This floor is purely
geometric, in nats; its \emph{value} depends on the cost unit, and its divergence at
certainty is invariant of it (\cref{nv4:thm:jauge}).
\end{corollaire}

\begin{remarque}[the cost-distance of evidence: a variance-limited reading]\label{nv4:rem:scaling}
Specialize the floor (\cref{nv4:cor:borne-inf}) to a single scale axis, at $e=0$, $U=2J$. On the
Gaussian leaf $J=1/\sigma^2$, so $\tilde g_{0,2J}=(4/\sigma^2)(\dd\mu^2+\dd\sigma^2)$ and the radial
cost element is $\dd d=2\,\dd\sigma/\sigma$; integrating from $\sigma_0$ down to $\sigma$,
\begin{equation}\label{nv4:eq:sigma-decay}
  d\ =\ \sqrt{2c}\,\ln\frac{\sigma_0}{\sigma},\qquad\text{equivalently}\qquad
  \boxed{\,\sigma\ =\ \sigma_0\,e^{-d/\sqrt{2c}}\,}\quad(c=2,\ \sqrt{2c}=2),
\end{equation}
saturating $d_0=\sqrt{2c}\,|\Delta H|$ along $\nabla(-H)$ (here $\Delta H=\ln(\sigma/\sigma_0)$). The
spread thus decays \emph{exponentially} in the cost-distance. Reading this $\sigma$ as the posterior
width of a parametric estimate places it in the \emph{variance-limited} regime: under
Bernstein--von Mises / Cramér--Rao, $\sigma^2\sim\sigma_{\mathrm{obs}}^2/n$ at evidence count $n$
(\cref{nv4:rem:bvm}), whence $-H\sim\tfrac12\ln n+\mathrm{const}$ and
\begin{equation}\label{nv4:eq:cost-log-n}
  d(n)\ =\ \frac{\sqrt{2c}}{2}\,\ln n+\mathrm{const},\qquad
  \mathrm{err}(d)=\sigma\sim n^{-1/2}\sim e^{-d/\sqrt{2c}} .
\end{equation}
The classical polynomial-in-$n$ error rate becomes exponential-in-cost-distance: each datum buys a
\emph{fixed} step of cost-distance (a fixed entropy drop), and the price of that step is constant: one data-doubling $n\to2n$ lowers $H$ by exactly $\tfrac12\ln2$ nats, i.e.\ a Landauer energy
$\tfrac12\ln2\,\kB T$ (\cref{nv4:rem:honnetete}, tier-P; a falsifiable constant marginal price). This
is the geometry \emph{transporting} the textbook $n^{-1/2}$ rate, not deriving it: the exponents are
those of variance-limited estimation \citep[exponent $1$ for the risk, $\tfrac12$ for the
error]{bahri2024}.

\emph{Scope: what this does not claim.} The empirical neural scaling laws
\citep{kaplan2020,hoffmann2022} have much smaller exponents ($\approx0.05$--$0.3$) and live in a
\emph{resolution-limited} regime, set by an intrinsic data dimension / feature spectrum
\citep{bahri2024} that the single location-scale leaf does not contain. We therefore claim no
bridge to those exponents, and none can be produced within this leaf: the curvature
$K=-1/(2cJ_0)$ is $\sigma$- and $n$-independent (\cref{nv4:thm:locscale,nv4:rem:bvm}) and $J_0$ only
rescales the prefactor through the gauge $c$ (\cref{nv4:thm:jauge}), so neither $K$ nor $J_0$ can
carry an $n$-exponent: curvature and evidence-count live on orthogonal axes. Recovering the
resolution-limited exponents would require a multi-dimensional / spectral extension of the conformal
class (a curved transport base, \cref{nv4:rem:C2}), which we leave open.
\end{remarque}

\paragraph{Existence without uniqueness.}
Completeness and coercivity give the existence of a minimizer, not its uniqueness.
The latter would require the geodesic convexity of the objective in the cost
metric, not established in general (\cref{nv4:sec:ouvert}). The framework establishes
well-posedness \emph{for existence and stability}; uniqueness remains open: the
non-convexity of possible interaction terms, and the multiplicity of minima
it may induce, are left open.

% =============================================================================
% §7  DISCUSSION
% =============================================================================
% =============================================================================
% §VI  DISCUSSION: SCOPE AND LIMITS
% =============================================================================
\section{Discussion: scope and limits}
\label{nv4:sec:discussion}\label{nv4:sec:ouvert}\label{nv4:sec:demarcation}

Three results are established: an infinite-distance boundary, an eikonal equivalence, a curvature bounded below
by Stam. It remains to situate them, with respect to what stays conjectural and to what
we do not claim (below); the positioning relative to neighboring geometries is
treated together with the related work (\cref{sec:travaux}).

\subsection*{What remains conjectural}

Three statements bound the scope of the theorems; we fix their exact status, without
broadening it.

\begin{remarque}[C1: beyond the conformal class]\label{nv4:rem:C1}
\cref{nv4:thm:carI}(i) (sufficiency) holds for every $U\in\Ccal$;
\cref{nv4:thm:carI}(ii) (necessity) is proved on the power family
$\{cJ^\alpha\}$. \textbf{Conjecture C1.} \emph{(a)} Necessity extends to every
$U\in\Ucal$ (any $U$ not dominating $J$ near $\partial_\infty$ leaves the
boundary at finite distance). \emph{(b)} Beyond conformal deformations
(\cref{nv4:def:candidat}), the characterization remains open. \emph{We never
claim uniqueness outside the conformal class.} \emph{(The choice of exponent $p=2$ within the
transport family is settled separately by \cref{nv4:prop:quad} and is orthogonal to this
conjecture.)}
\end{remarque}

\begin{remarque}[C2: curvature beyond the Gaussian: resolved on location-scale]\label{nv4:rem:C2}
\textbf{Resolved (location-scale).} Beyond the Gaussian leaf,
\cref{nv4:thm:locscale,nv4:thm:jauge} establish that on \emph{every} 1D location-scale family the
curvature, at $e=0$, equals $-1/(2cJ_0)<0$: \emph{negativity is robust}, and the extremal value is attained by the Gaussian (Stam, $J_0=1$).
\textbf{Conjecture C2 (what remains open).}
\emph{(a)} Negativity near $\partial_\infty$ persists on 2D families
that are \emph{not} location-scale (curved $\Wtwo$ base). \emph{(b)} In dimension $\ge2$ /
on all of $\Ptwo$, the \emph{ambient sectional} curvature of $\tilde g_{e,2J}$
is negative near $\partial_\infty$, a heavier object: Otto computation $+$
\emph{dimension-dependent} conformal transformation. \emph{Tooling
note:} the curvature-dimension condition $CD(K,N)$ (here $K$ is a generic Ricci lower bound, not the Gauss curvature) \citep{lottvillani2009,sturm2006} bounds the Ricci of the \emph{base}
$(\Ptwo,\Wtwo)$ via convexity of the entropy; this is \emph{not} the curvature of
the cost metric $\tilde g$ (a distinct conformal object); $CD(K,N)$ therefore
applies only to part \emph{(b)}, not to \cref{nv4:thm:locscale}. \emph{We do not
conflate the intrinsic family curvature (\cref{nv4:thm:locscale}) and the
ambient sectional one.}
\end{remarque}

\begin{remarque}[C3: global existence on $\Ptwo$]\label{nv4:rem:C3}
\cref{nv4:thm:exist} gives existence for the action; the existence of the minimizer
of an estimate on all of $\Ptwo(\Rbb^n)$ (mass escape) requires
concentration-compactness \citep{lions1984}, beyond the reach of a properly proved
statement as it stands. \emph{Status:} scoped; the interior minimizer guaranteed by the boundary
(\cref{nv4:thm:carI}) suffices for fixed endpoints / coercive fidelity.
\end{remarque}

\subsection*{What is not claimed}

The framework is a conditional characterization; its scope is bounded as follows.
\begin{itemize}\itemsep2pt
\item P0 is a posed modelling choice: revision is treated as optimal transport, an adopted arena
  (anisotropic and non-transport alternatives stay open).
\item P1 is one calibration among several: uniform pricing is adopted, with an energy- or
  divergence-proportional price equally available.
\item Landauer anchors the unit; the metric rests on the two postulates (\cref{nv4:prop:firewall}).
\item Uniqueness holds within the conformal class, among continuous reliefs; the curvature \emph{reads}
  the Fisher information, and $-\tfrac14$ is the image of a unit (\cref{nv4:rem:representant}).
\item ``Well-posed'' covers existence and stability; uniqueness of the minimizer stays open.
\item The geometry prices revisions; whether an agent's trajectory follows its geodesics is a separate,
  conjectural question (\cref{nv4:sec:perspectives}).
\item The invariance claimed is metrological, units of cost and of state
  (\cref{nv4:thm:jauge,nv4:prop:etat-unites}); statistical-morphism invariance is \v{C}encov's
  (\cref{nv4:rem:amari}).
\item General necessity, curvature universality, and global existence on $\Ptwo$ stay the conjectures
  C1--C3 (\cref{nv4:rem:C1,nv4:rem:C2,nv4:rem:C3}).
\end{itemize}

% =============================================================================
% §7bis  RELATED WORK
% =============================================================================
\section{Related work}
\label{sec:travaux}

Each building block is cited where it is used; we gather here the boundary
between what is acquired and what is contributed, complementing the three
positioning remarks (\cref{nv4:rem:amari,nv4:rem:ito,nv4:rem:friston}). The
contribution amounts to two moves: the eikonal characterization $U=cJ$
(\cref{nv4:thm:carII}) and the assembly into a cost metric with an infinite-distance boundary; the rest
is borrowed.

\paragraph{Geometry of inference.}
Information geometry \citep{chentsov1972,amari2016} equips beliefs with the
Fisher--Rao metric, the only one, up to scale, invariant under statistical morphisms. Otto's
calculus \citep{otto2001} makes the complementary move: it endows $\Ptwo$ with a
Riemannian structure of distance $\Wtwo$, where the heat flow is the gradient
flow of entropy \citep{jko1998,ambrosio2008} and where
$\lVert\nabla_{\Wtwo}H\rVert^2=J$ \citep{villani2009}, the identity that makes
$J$ the \emph{slope} of $-H$ (\cref{nv4:ax:B1}), the metric form of de Bruijn's classical identity $\tfrac{d}{dt}H(X_t)=\tfrac12 J(X_t)$ along the heat flow \citep[Th.~17.7.2]{coverthomas2006}. The HWI (entropy--Wasserstein--Fisher)
inequality \citep{ottovillani2000} already relates our objects in one line. Our
framework characterizes \emph{another} class, the conformal deformations of
transport, whose curvature separates it from Fisher--Rao on \emph{structural} grounds, position information with a Stam (Gaussian) extremum, versus scale information, the value $-\tfrac14$ being a gauge image
(\cref{nv4:rem:amari,nv4:rem:locscale-pos}).

\paragraph{Thermodynamics of information.}
Erasing one nat dissipates at least $\kB T$ \citep{landauer1961,bennett1982}; the
thermodynamics of information turned this into an exact bookkeeping
\citep{parrondo2015}: extractable work bounded by the information acquired
\citep{sagawa2010}, and an \emph{energetic} reading of geodesics where the path
of minimal dissipation is optimal transport \citep{aurell2011,dechant2019}. This
is the currency, nat for nat, and the legitimacy of treating an inference device
as a thermodynamic system; the energetic optimality of geodesics is their
theorem, not ours. Our cost is \emph{derived} from this price.

\paragraph{Finite-time dissipation.}
Benamou--Brenier \citep{benamou2000} gives $\Wtwo^2$ a kinetic reading;
stochastic thermodynamics turns it into a price: the dissipation of a
diffusion driven from $p_0$ to $p_1$ in time $T$ is bounded below by
$\Wtwo^2(p_0,p_1)/T$ \citep{dechant2019}, and the thermodynamic length
\citep{sivak2012} is its linear regime; a second family bounds the \emph{rates}
by the Fisher \emph{of the path} \citep{ito2018}. These works \emph{measure} the
dissipation on a given geometry; they do not re-metrize the space of beliefs nor
place an infinite-distance boundary in it, and their Fisher is that of the path, not of the state $J(p)$
(\cref{nv4:rem:ito}).

\paragraph{Fisher cost in learning.}
This is a direct link. Elastic weight consolidation
\citep{kirkpatrick2017} fights forgetting by penalizing the displacement of a
belief \emph{in weight space} by a \emph{squared Fisher distance} toward a reference state, derived from
a Laplace approximation of the posterior \citep{huszar2018} (Laplace's method, the Gaussian
posterior near the mode \citep[p.~341, and in learning p.~501]{mackay2003}): on the Gaussian
leaf, a kindred Fisher cost, the same principle, but \emph{static}, a frozen
anchor, where the maintenance cost is dynamic. \citet{still2012}
price \emph{non-predictive} information under driving and
\citet{kolchinsky2018} ground semantic information thermodynamically; here
even useful information has a permanent maintenance cost, because in our setting the world diffuses.
The natural gradient \citep{amari2016} preconditions by the Fisher, a relative of
well-posedness (\cref{nv4:cor:existence}).

\paragraph{Mixed Wasserstein--Fisher geometries, and positioning.}
\citet{vonrenesse2012} reads the Schrödinger equation as a Newton law on
$\Ptwo$ with potential the Fisher, and \citet{conforti2018} classify transport,
Schrödinger bridge and Madelung fluid as extremals of actions that differ only by
the \emph{sign} of a Fisher term; to be distinguished from the
Wasserstein--Fisher--Rao metric \citep{chizat2018}, which
interpolates \emph{additively} by varying the mass, the opposite of our conformal
deformation at mass $1$. The conformal (multiplicative) reweighting of the
transport metric itself has a precedent: \citet{ambrosiosanta2007} multiply the
Wasserstein metric by a scalar factor (a power of an $L^q$ norm of the density)
to favour spreading, with no curvature; here the factor is the Fisher information,
from which $\partial_\infty$, the curvature and the Stam extremality follow. Concurrent works (2024--2026) treat neighboring facets
without covering this assembly: \citet{okanohara2026} bounds the dissipation
of learning by an ensemble Wasserstein distance (transport of models,
a relative of that of \citet{dechant2019}, where the maintenance cost bears on the \emph{state}
of Fisher $J(p)$), \citet{melo2025} study the entropy production tied to
Fisher information in stochastic thermodynamics, of which \citet{itosagawa2016}
gives the Bayesian-network setting, and \citet{hyland2025} price the cost of
\emph{changing} belief, where $\partial_\infty$ (\cref{nv4:thm:carI}) prices that of
\emph{holding} it. To our knowledge, the eikonal characterization $U=cJ$ and
the metric inaccessibility of certainty as an infinite-distance boundary on the conformal class have not
been formulated as such.

% =============================================================================
% §8  CONCLUSION
% =============================================================================

\subsection*{Positioning}

The characterization is read first by contrast, opposing three currencies (Fisher--Rao, bare transport and uniform-pricing cost) on the same location-scale leaf. The
principal contrast is with Fisher--Rao, and it sharpens into a thesis \emph{parallel} to
\v{C}encov's:
\begin{quote}\itshape
Fisher--Rao is the canonical geometry of distinguishability, fixed by invariance under statistical
morphisms (\v{C}encov). The Fisher-reweighted transport metric is the geometry uniform pricing characterizes for
revision cost: theorem-grade on the power family, conjectural in general (Conjecture~C1,
\cref{nv4:rem:C1}). \v{C}encov's selector is an invariance; ours is a posed (uniform) pricing.
\end{quote}
The two uniqueness statements answer different questions (how distinguishable two beliefs are,
versus what it costs to move between them); as with \v{C}encov, the uniqueness here is relative to
its class (\cref{nv4:rem:amari}). The invariances differ accordingly: to \v{C}encov's statistical-morphism invariance
corresponds here a metrological one, under the units of cost (\cref{nv4:thm:jauge}) and of the
state (\cref{nv4:prop:etat-unites}). \cref{tab:arenes}
makes the contrast checkable at a glance: the four \emph{discriminating} capabilities live in
$\Wtwo$ alone. The three remarks that follow unfold the argument, each against an established
program.\footnote{The location-scale leaves are the noise models of M-estimators
($\varphi_0\leftrightarrow$ loss $\rho=-\log\varphi_0$); the curvature $K=-1/(2cJ_0)$
(\cref{nv4:thm:locscale,nv4:thm:jauge}) orders them, the Gaussian (least squares, $J_0{=}1$) being the Stam
extreme (a bridge to robust estimation \citep{huber1981} left for elsewhere).}

\begin{table}[!t]
\centering\small
\caption{\textbf{What each arena can and cannot state.} Rows: capabilities the characterization
needs; columns: candidate arenas (WFR: Wasserstein--Fisher--Rao).
\checkmark\ available; $\times$ absent; $\sim$ present but in a form that changes the object. The
four \emph{discriminating} rows (bold) are $\times$ everywhere but in $\Wtwo$. The $\partial_\infty$-distance
discrimination is theorem-grade on the power family $\{cJ^\alpha\}$; general necessity is
Conjecture~C1 (\cref{nv4:rem:C1}).}
\label{tab:arenes}
\begin{tabular}{@{}l ccccc@{}}
\toprule
Required capability & Fisher--Rao & KL & $W_{p\neq2}$ & WFR & $\Wtwo$ (ours) \\
\midrule
Genuine metric (distance, geodesics) & \checkmark & $\times$ & \checkmark & \checkmark & \checkmark \\
Riemannian structure $+$ curvature & \checkmark & $\times$ & $\times$ & \checkmark & \checkmark \\
\textbf{$J$ as slope of $-H$ (Otto)} & $\times$ & $\times$ & $\times$ & $\times$ & \checkmark \\
\textbf{$\partial_\infty$-distance \emph{discriminates} on $U$} & $\times$ & $\times$ & $\times$ & $\times$ & \checkmark \\
\textbf{Eikonal $U{=}cJ$ has \emph{content}} & $\times$ & $\times$ & $\times$ & $\times$ & \checkmark \\
\textbf{Stam $\Rightarrow$ Gaussian extremum} & $\times$ & $\times$ & $\times$ & $\times$ & \checkmark \\
Mass conserved (probability $=1$) & \checkmark & \checkmark & \checkmark & $\times$ & \checkmark \\
Dissipation grounding (physical cost) & $\times$ & $\times$ & $\sim$ & $\sim$ & \checkmark \\
\bottomrule
\end{tabular}

\smallskip
{\footnotesize\itshape Why the crosses (Fisher--Rao unless noted): KL is asymmetric (no metric);
$W_{p\neq2}$ has no Riemannian/Otto structure; Fisher--Rao reads $J$ as a \emph{component} of the metric tensor, so the eikonal there is vacuous: its content comes only from the
Otto slope of $-H$, which the transport arena supplies; $\partial_\infty$ is automatically at infinite distance
(the Dirac is already at infinite distance, nothing to discriminate); and Stam bounds the
\emph{position} information, whereas Fisher--Rao location-scale curvature reads the \emph{scale}
information (symmetric base; see \cref{nv4:rem:locscale-pos}), which Stam does not bound.
WFR varies the mass, changing the object.}
\end{table}

\begin{remarque}[vs information geometry (Amari--\v{C}encov)]\label{nv4:rem:amari}
\v{C}encov characterizes Fisher--Rao as \emph{the unique} metric, up to scale, invariant under
statistical morphisms \citep{chentsov1972,amari2016}. Our characterization is
of analogous nature but in \emph{another class}: conformal deformations of
\emph{transport} (Wasserstein), where $J$ is the \emph{slope} of $-H$
(\cref{nv4:ax:B1}): the recasting (\cref{nv4:rem:dualJ})
read once more, here in the positioning itself. Distinctive signature: the curvature reads the Stam-bounded \emph{position} information, whereas Fisher--Rao reads the \emph{scale} information; the value
$-\tfrac14$ is its image at the representative gauge $c=2$ (\emph{vs} $-\tfrac12$ for Fisher--Rao)
\citep{atkinson1981,costa2015}. On location-scale families, \emph{bare} transport has curvature $\ge0$ \citep{takatsu2011}, where the
\emph{uniform-pricing} cost makes the \emph{position} information $J_0$ carry the curvature, $J_0$ being
bounded below by Stam (\cref{nv4:thm:locscale,nv4:rem:locscale-pos}); the negativity arises from the
conformal deformation, not from bare transport. Deforming a base metric to
negative curvature on location-scale families is a known template:
warped Fisher--Rao \citep{said2017}, the $\lambda$-deformation to a conformal
Hessian metric of constant curvature \citep{zhangwong2022}, the Hessian metric of
an information functional on $\Ptwo$ \citep{wuchenli2020}. The closest of these,
\citet{wuchenli2020}, builds the \emph{Hessian} of an information functional (a full
second-derivative tensor) on $\Ptwo$, generically not conformal to $g_{\Wtwo}$;
here the Fisher enters as a \emph{scalar} conformal factor $2(e+U)\,g_{\Wtwo}$ on the
transport metric, and none of \citet{said2017,zhangwong2022,wuchenli2020} yields the
infinite-distance boundary, the eikonal characterization $U=cJ$, or the Stam extremality. The same line separates
us from the \emph{Wasserstein natural gradient} and \emph{Wasserstein information matrix}
\citep{limontufar2018,lizhao2023}, which pull the transport metric back as a \emph{tensor}
preconditioner on a finite-dimensional parameter space (with a Wasserstein--Cram\'er--Rao bound on
location-scale families), and from the \emph{Wasserstein statistics} of \citet{amarimatsuda2024} on
those same affine families, which use \emph{bare} transport (no conformal $J$-reweighting) and
so yield neither an infinite-distance boundary, nor negative curvature, nor Stam extremum. \emph{As with \v{C}encov, the uniqueness
(here of family) is relative to the class; we make it explicit.}
\end{remarque}

\begin{remarque}[vs geometric thermodynamics (Ito)]\label{nv4:rem:ito}
Stochastic thermodynamics establishes the Fisher$\leftrightarrow$Wasserstein bridge
\citep{sivak2012,aurell2011,ito2023}: entropy production is bounded by $\Wtwo$ (a speed limit), $J$
the excess entropy rate. These works \emph{measure} dissipation on a given geometry; we
\emph{characterize} a class of conformal \emph{cost} metrics with an infinite-distance boundary and a curvature.
\end{remarque}

\begin{remarque}[vs Bayesian mechanics (Friston)]\label{nv4:rem:friston}
The free-energy principle carries ``dissipative agent $\to$ geometry of belief'' on a Fisher
manifold \citep{friston2019,sakthivadivel2022}, without an infinite-distance boundary, negative curvature driven by
\emph{position} information, or an eikonal $U=cJ$ on Wasserstein. This work's contribution is that
\emph{assembly} (eikonal and an infinite-distance boundary on transport) more than each ingredient alone.
\end{remarque}

\section{Perspectives}
\label{nv4:sec:perspectives}

The preceding characterizations concern a \emph{background} cost geometry:
fixed by uniform pricing (\cref{nv4:thm:carII}) and determined up to a unit
(\cref{nv4:thm:jauge}), the metric $\tilde g_{e,U}=2(e+U)\,g_{\Wtwo}$ is a scene
given in advance. As an analogy for the open directions below, a belief moves on it like a test
particle; the metric prices revisions, and whether an agent follows its geodesics stays a separate,
conjectural question. A
\emph{self-consistent} geometry, sourced by the configuration of beliefs it carries rather
than prescribed, could extend it. The displacement would then bear on the
scene itself, beyond the cost unit fixed by gauge invariance
(\cref{nv4:thm:jauge}); we borrow from general relativity, as an analogy, its
lesson of background independence~\citep{einstein1916}. One
ingredient already lends itself to this: in the kinetic action
$\tfrac12|\dot\gamma|^2_{\tilde g}=(e+cJ)\,|\dot\gamma|^2_{\Wtwo}$, the factor $(e+cJ)$
would play the role of an \emph{inertial mass}: the resistance to changing one's mind,
large near $\partial_\infty$ (\cref{nv4:thm:carI}). It remains to define the \emph{active} mass of a
belief (by which it could, in return, contribute to the geometry seen
by other beliefs) and its coupling. The three statements below, in the conditional, mark out the way\footnote{C4, C5 and C6
follow the numbering style of conjectures C1--C3
(\cref{nv4:rem:C1,nv4:rem:C2,nv4:rem:C3}); unlike those, which bound
the scope of the established theorems, C4--C6 open directions outside the proof
structure.}; following the convention of the paper, they \emph{motivate} and enter into
no proof.

\begin{remarque}[C4: active mass and inertial mass]\label{nv4:rem:C4}
In a self-consistent completion of the cost geometry (where the relief $U$ would be
sourced by the belief configuration), the active mass of a belief
would coincide with its inertial mass $\propto(e+cJ)$, already present as the factor of
$|\dot\gamma|^2$ in $\tilde g$.
\end{remarque}

\begin{remarque}[C5: toward an epistemic hysteresis]\label{nv4:rem:C5}
A nonconvex attractive interaction term between correlated beliefs would induce a
multiplicity of minima; under a quasi-static dynamics, the belief would exhibit a
\emph{hysteresis}, whose loop width would grow with the inertial mass
($\propto e+cJ$) and the coupling. The two-route first-order transition signaled in
\cref{nv4:thm:principal} would be its minimal instance.
\end{remarque}

\begin{remarque}[C6: relativized pricing: toward a self-consistent geometry]\label{nv4:rem:C6}
Postulate~1 fixes a \emph{globally} constant price per nat ($\kappa$ const $\Rightarrow U=cJ$).
Relaxing it to a \emph{local} price, $\kappa=\kappa(p)\Rightarrow U=c(p)\,J$, turns the cost unit
into a \emph{field} over the space of beliefs, keeping the pointwise proportionality $U\propto J$
(only the coefficient $c$ becomes a field). This is the field-over-configuration counterpart of the
self-consistent geometry of \cref{nv4:rem:C4}: relativizing loses uniqueness ($c(p)$ is a free
function) unless a field equation determines it from the belief configuration, which is precisely
the coupling left undefined in \cref{nv4:rem:C4}; hence a well-posed open problem, not a result. A
natural source already present is the plausibility $\Pi\propto\nu\,\exp(-U)$
(\cref{nv4:def:plausibilite}). The heterogeneity it produces is \emph{configurational} (cheaper and
dearer regions of belief space, hence agent-dependent revision), not a dynamics in time: it leaves
the static reading (\cref{nv4:rem:bvm}) untouched.
\end{remarque}

The development of these directions (self-consistent geometry, interaction
dynamics, hysteresis) belongs to separate work.
\section{Conclusion}
\label{nv4:sec:conclusion}

A finite agent's belief stops short of certainty (\cref{sec:cadre}).
On the conformal class, $\partial_\infty$ (certainty pushed off to infinite distance) is
\emph{characterized} by the dominance of Fisher information (sufficient
unconditionally; necessary at the boundary, general case conjectured;
\cref{nv4:thm:carI}); there the eikonal characterizes, among continuous reliefs, the Fisher family $U=cJ$
(\cref{nv4:thm:carII}), and the curvature $K=-1/(2cJ_0)$ at $e=0$, \emph{over any
location-scale family}, is extremal at the Gaussian, by Stam rigidity
(\cref{nv4:thm:hyper,nv4:thm:locscale}). These three characterizations are the
invariants of one and the same change of cost unit (\cref{nv4:thm:jauge}), and a global change
of state units leaves them unchanged (\cref{nv4:prop:etat-unites}): thermodynamics anchors the
unit, and the proven conclusions hold up to it.
The contribution is a \emph{characterization}: two explicit modelling commitments,
P0 and P1, pin down a single geometry of belief-cost, within the conformal class,
over which inference is well-posed and reaching a precision has a floor diverging at
certainty. Removing either commitment leaves the selection open. Its extension (a
geometry sourced by the beliefs it carries) is sketched in \cref{nv4:sec:perspectives}.

% =============================================================================
% APPENDIX A  — Proofs of the machinery
% =============================================================================
\appendix
\section{Proofs of the structural results}
\label{nv4:app:machine}

We prove here the class-machinery statements announced in \cref{nv4:sec:machine},
the AM--GM inequality (\cref{nv4:lem:amgm}), the Maupertuis--Jacobi
correspondence (\cref{nv4:thm:maupertuis}), the existence of geodesics (\cref{nv4:thm:exist}), the energy
law (\cref{nv4:prop:energie}) and the equivalence optimal curve $\Leftrightarrow$
geodesic (\cref{nv4:thm:principal}), valid for every candidate $U\in\Ucal$; the
regime clauses that bound them (duality gap $e>0$, du Bois-Reymond lemma, compactness
of the sublevels by Arzelà--Ascoli) enter explicitly in each proof.

\begin{proof}[Proof of \cref{nv4:lem:amgm}]
$\tfrac12a^2+b-\sqrt{2b}\,a=\tfrac12(a-\sqrt{2b})^2\ge0$, zero iff $a=\sqrt{2b}$.
\end{proof}

\begin{proof}[Proof of \cref{nv4:thm:maupertuis}]
\emph{(i)} \cref{nv4:lem:amgm} with $a=\md\gamma$, $b=e+U$, integrated.
\emph{(ii)} reparametrize at constant $\Wtwo$ speed $L$ \citep[lem.~1.1.4]{ambrosio2008},
set $t(s)=\int_0^s L/\sqrt{2(e+U(\hat\gamma_r))}\dd r$ (integrand $\le L/\sqrt{2e}$,
$>0$ a.e.); the equality case of \cref{nv4:lem:amgm} concludes.
\emph{(iii)} equality in \emph{(i)} = pointwise equality case.
\end{proof}

\begin{proof}[Proof of \cref{nv4:thm:exist}]
Minimizing sequence $\Acal_T(\gamma^k)\le C$; since $U\ge0$, \cref{nv4:lem:K}(a) gives
$\Ecal(\gamma^k)\le2\Acal_T(\gamma^k)\le2C$, whence the Hölder bound
$\Wtwo(\gamma^k_s,\gamma^k_t)\le\sqrt{\Ecal(\gamma^k)\,|t-s|}\le\sqrt{2C\,|t-s|}$. A uniform
second-moment bound follows: as $\sqrt{m_2(\mu)}=\Wtwo(\mu,\delta_0)$,
$\sqrt{m_2(\gamma^k_t)}\le\Wtwo(\gamma^k_t,p_0)+\Wtwo(p_0,\delta_0)\le\sqrt{2CT}+\sqrt{m_2(p_0)}=:\sqrt M$,
so the curves stay in the $\tau$-compact $\{m_2\le M\}$ (\cref{nv4:ax:B3}), on which $\tau$ is
metrized by the Dudley metric $\beta_0\le\mathcal W_1\le\Wtwo$ ($\mathcal W_1$ the $1$-Wasserstein distance); the $\Wtwo$-Hölder bound is a
fortiori $\beta_0$-equicontinuity, so Arzelà--Ascoli extracts a $\tau$-limit $\gamma$. By
$\tau$-l.s.c.\ of $\Ecal$ (\cref{nv4:lem:K}(c)) and Fatou for $\int U$ ($U$ l.s.c.), one gets
$\Ecal(\gamma)\le2C<\infty$ (hence $\gamma\in AC^2$, \cref{nv4:lem:K}(b)) and
$\Acal_T(\gamma)\le\liminf_k\Acal_T(\gamma^k)=\Phi(T)$.
\end{proof}

\begin{proof}[Proof of \cref{nv4:prop:energie}]
Inner variations: for $u:[0,T]\to(0,\infty)$, $u,1/u\in L^\infty$,
$\int u=T$, with $\theta_u(t):=\int_0^t u$, $\gamma^u_t:=\gamma^\star_{\theta_u(t)}$ is admissible and
$\Acal_T(\gamma^u)=\int[\tfrac12\md{\gamma^\star}^2/u+U u]=:F(u)$. $F$ convex,
$u\equiv1$ optimal. $V(s):=\inf\{F:\int u=s\}$ convex;
$\mathrm{dom}\,V=(0,\infty)$ (for every $s>0$, $u_s:=s/T$ gives
$\int u_s=s$ and $F(u_s)=\tfrac{T}{s}\!\int\tfrac12\md{\gamma^\star}^2
+\tfrac{s}{T}\!\int U(\gamma^\star)<\infty$), so $T$ is \emph{interior} there and
$\partial V(T)\ne\emptyset$; taking $-e\in\partial V(T)$, the first
variation $\int[-\tfrac12\md{\gamma^\star}^2+(U+e)]h=0\ (\forall h\in L^\infty)$
gives (du Bois-Reymond \citep{buttazzo1998}) the stationarity, and the convexity of
$u\mapsto\tfrac12\md{\gamma^\star}^2/u+(U+e)u$ ($\partial_u^2\ge0$) makes it the
pointwise minimizer.
\end{proof}

\begin{proof}[Proof of \cref{nv4:thm:principal}]
The energy law gives the pointwise equality of \cref{nv4:lem:amgm}, hence
$\Acal_T(\gamma^\star)+eT=\ell_e(\gamma^\star)$ and the optimality of the trace for
$\ell_e$ at parameter $T_e=T$. Write $G(x):=\ell_x(\gamma^\star)-xT$ and
$D(x):=\Psi(x)-xT$; since $\Psi\le\ell_\cdot(\gamma^\star)$ pointwise, $D\le G$.
$G$ is strictly concave ($G''<0$) and $G'(e)=\int\md{\gamma^\star}/\sqrt{2(e+U)}-T=0$
by the energy law ($e>0\Rightarrow\md{\gamma^\star}=\sqrt{2(e+U)}>0$ a.e.), so $e$ is the
\emph{unique} maximizer of $G$, with $G(e)=\ell_e(\gamma^\star)-eT$.
$D$ is concave with $D(x)\to-\infty$ as $x\to+\infty$, so $\sup_x D$ is attained, say at
$x^\star$; exact duality gives $\Phi(T)=D(x^\star)\le G(x^\star)\le G(e)$, whence, with
$\Phi(T)=\Acal_T(\gamma^\star)=\ell_e(\gamma^\star)-eT=G(e)$, equality throughout and, by strict
concavity, $x^\star=e$. Therefore $\Psi(e)=D(e)+eT=\ell_e(\gamma^\star)$
(global geodesic). Converse: \cref{nv4:thm:maupertuis}(ii).
\end{proof}

% =============================================================================
% APPENDIX B  — Dependency table and logical independence of the proofs
% =============================================================================
\section{Dependency table and logical independence of the proofs}
\label{nv4:app:dag}\label{nv4:sec:dag}

\cref{nv4:tab:dag} records, statement by statement, the authoritative direct logical
dependencies; it is the only place where the provenances P (physical/interpretive), E
(borrowed) and D (proved) are labeled, following the convention recalled in the table's
note. The acyclicity proposition below establishes that no proved statement depends on a
physical statement.

\begin{table}[!t]
\centering\small
\caption{\textbf{Dependency table} (authoritative). Provenances:
\tierP{} physical/interpretive (motivates, never a proof); \tierE{} borrowed
(standard theorem cited); \tierD{} proved here, in pure metric geometry; \emph{posed} = defended postulate (P0/P1); \emph{setting} = modeling frame (definitions/axioms).}
\label{nv4:tab:dag}
\begin{tabular}{@{}llp{0.46\textwidth}@{}}
\toprule
Statement & Tier & Direct logical dependencies \\
\midrule
P: physical (Landauer/Bekenstein) & \tierP & none \emph{(motivates, dashed: Landauer fixes the unit; the dissipation floor $\Wtwo^2/T$ corroborates \cref{nv4:ax:P0}; operational constraints, \cref{nv4:rem:motive})} \\
P0, P1 (postulates) & posed & none \emph{(defended choices; P1 presupposes P0, via Char.\,II; enter as hypotheses, no incoming edge)} \\
A1, A2; B1--B7 & setting/\tierE & none \\
A3 arena (\cref{nv4:ax:arene}) & setting & \cref{nv4:ax:P0} (transport type); \cref{nv4:prop:quad} ($p{=}2$) \\
\cref{nv4:lem:amgm}, \cref{nv4:lem:K} & \tierD & none; B3 \\
\cref{nv4:lem:slope} & \tierE & B1, B2 \\
Maupertuis (\cref{nv4:thm:maupertuis}) & \tierD & \cref{nv4:lem:amgm} \\
existence (\cref{nv4:thm:exist}) & \tierD & \cref{nv4:lem:K}, B3 \\
energy law (\cref{nv4:prop:energie}) & \tierD & \cref{nv4:thm:exist} \\
geodesic (\cref{nv4:thm:principal}) & \tierD & \cref{nv4:thm:maupertuis}, \cref{nv4:prop:energie} \\
\textbf{Char.\,I} $\partial_\infty$ at infinite distance (\cref{nv4:thm:carI}) & \tierD & \cref{nv4:def:classe}, \cref{nv4:lem:slope}, B4, \cref{nv4:thm:exist} \emph{(B1 only transitively, via \cref{nv4:lem:slope}; carI uses only the B2 inequality)} \\
\textbf{Char.\,II} eikonal (\cref{nv4:thm:carII}) & \tierD & B1 (Otto/A2), \cref{nv4:def:candidat} (conformity), \cref{nv4:def:eikonal} (eikonal condition, used as hypothesis: the content P1 posits; Char.\,II is a conditional equivalence, not assuming P1 true, cf.\ \cref{nv4:prop:firewall}) \\
admissibility $\Fcal$ (\cref{nv4:prop:meta}) & \tierD & B6 \\
hyperbolicity (\cref{nv4:thm:hyper}) & \tierD & \cref{nv4:def:fisher}, B4 \\
Stam extremum (\cref{nv4:thm:locscale}) & \tierD & B7, B6 (l.s.c.\ dual); Stam extremality (\citealp{stam1959}; CT Ch.~17) \\
$p=2$ forced (\cref{nv4:prop:quad}) & \tierD & B7, B4, Jordan--von Neumann (cited) \\
representative $2J$ (\cref{nv4:rem:representant}) & \tierD & \cref{nv4:thm:carII}, \cref{nv4:thm:hyper}, \cref{nv4:thm:locscale} \\
setting defs.\ & setting & \cref{nv4:def:candidat,nv4:def:classe,nv4:def:mur,nv4:def:eikonal,nv4:def:fisher} \emph{($e$ and the ``$2$'' enter as a unit; cf.\ \cref{nv4:thm:jauge})} \\
gauge invariance (\cref{nv4:thm:jauge}) & \tierD & \cref{nv4:def:candidat}, \cref{nv4:thm:locscale} \\
state units (\cref{nv4:prop:etat-unites}) & \tierD & \cref{nv4:def:candidat}, \cref{nv4:ax:B7} (leaf clause); elementary scalings of $H$, $J$, $\Wtwo$; invariance asserted of \cref{nv4:thm:carI,nv4:thm:carII,nv4:thm:locscale} \\
finiteness cor.\ (\cref{nv4:cor:finitude}) & \tierD/\tierP & \cref{nv4:thm:carI}; \emph{motivates}: P (dashed) \\
uniform pricing (\cref{nv4:rem:honnetete}) & \tierD/\tierP & \cref{nv4:thm:carI}, \cref{nv4:thm:carII}; \emph{motivates}: P (dashed) \\
base (\cref{nv4:def:arene}, \cref{nv4:def:R}) & setting & none \emph{(object: $H$, $J$, $\partial_\infty$)} \\
existence (\cref{nv4:cor:existence}) & \tierD & \cref{nv4:thm:carI}, \cref{nv4:thm:exist}, \cref{nv4:ax:B3} \\
cost--entropy (\cref{nv4:cor:descente}) & \tierD & \cref{nv4:thm:carI}(i) \\
cost floor (\cref{nv4:cor:borne-inf}) & \tierD/\tierP & \cref{nv4:thm:carI}(i); \emph{motivates}: P (dashed) \\
\bottomrule
\end{tabular}
\end{table}

\begin{proposition}[acyclicity of dependencies: no proved statement depends on a physical statement \tierD]\label{nv4:prop:firewall}
Let $\to$ denote direct logical dependence (col.~3 of \cref{nv4:tab:dag}, which is
authoritative; set-theoretic reading: P outside the universe). Here \tierP{} \emph{physical}
(Landauer/Bekenstein \citep{bekenstein1981}, dissipation floor) is not the \emph{postulates} P0, P1: the latter are posed
hypotheses (setting tier) and, like every setting item, may have outgoing edges: promoting the
eikonal to Postulate~1 adds the edge P1\,$\to$\,Char.\,II (which later \tierD{} items, e.g.\
\cref{nv4:rem:representant,nv4:rem:honnetete}, do consume); this is harmless because P1 is a posed
setting-tier hypothesis, not the \tierP{} physical statement~P. The logical separation below concerns \tierP{}
alone: part~(ii) forbids only solid edges carrying~P, and no col.~3 entry lists~P (it appears only in
the dashed ``motivates'' clause).
Then:
\emph{(i)} the graph on $\{P\}\cup\{\text{setting}\}\cup\{\text{borrowed results }\tierE\}\cup\{\text{statements }\tierD\}$
is acyclic (stratification by rank: each solid edge goes from a lower rank to a higher
rank);
\emph{(ii)} $\forall D\ \tierD:\ P\notin\mathrm{cl}_\to(D)$: P has no outgoing
\emph{solid} edge (the ``motivates'' col.\ $=$ dashed edges only), so every
$D$ is proved from setting $+$ borrowed results alone. Each edge of
\cref{nv4:tab:dag} has been traced back in the body of the corresponding proof
(finite verification).
\end{proposition}
\begin{proof}
By inspection of \cref{nv4:tab:dag}: no ``logical dependencies'' column
contains P; the rank
($\text{rank}(D)=1+\max\{\text{rank}(x):x\to D\}$) is well-defined and finite, whence
acyclicity; induction on the rank for $P\notin\mathrm{cl}_\to(D)$.
\end{proof}

\section{Notation}
\label{nv4:app:notation}
\begin{center}\small
\begin{tabular}{@{}p{0.34\textwidth}p{0.60\textwidth}@{}}
\toprule
Symbol & Meaning (defined at) \\
\midrule
$\Ptwo(\Rbb^n)$ & beliefs: densities of finite second moment (\cref{nv4:def:arene}) \\
$\Wtwo,\ g_{\Wtwo}$ & Wasserstein-2 distance and transport metric (\cref{nv4:ax:P0}) \\
$H,\ -H{=}\Ent$ & Shannon entropy and knowledge (\cref{nv4:def:arene}) \\
$J$ & Fisher information $\int|\nabla p|^2/p$ (\cref{nv4:def:arene}) \\
$J_0$ & Fisher information of the unit-variance base (\cref{nv4:thm:locscale}) \\
$\partial_\infty$ & boundary of certainties $\{H{=}{-}\infty\}$ (\cref{nv4:def:arene}) \\
$U$ & relief: local price of precision (\cref{nv4:def:plausibilite}) \\
$e$ & baseline level, the energy constant (\cref{nv4:def:candidat,nv4:prop:energie}) \\
$c$ & cost unit (gauge; $c{=}2$ representative) (\cref{nv4:thm:jauge}) \\
$\tilde g_{e,U}$ & cost metric $2(e+U)\,g_{\Wtwo}$ (\cref{nv4:def:candidat}) \\
$d_e$ & cost distance, geodesic length of $\tilde g_{e,U}$ (\cref{nv4:def:candidat}) \\
$\kappa$ & metric slope of $-H$ at $e{=}0$, $\kappa=1/\sqrt{2c}$ (\cref{nv4:def:eikonal}) \\
$K$ & Gauss curvature of the leaf (\cref{nv4:thm:hyper}) \\
$\mathcal R$ & regularity class, the interior domain (\cref{nv4:def:R}) \\
$\Fcal\subset\Ccal\subset\Wcal\subset\Ucal$ & Fisher, well-posed, boundary, candidate relief classes (\cref{nv4:def:candidat}) \\
\bottomrule
\end{tabular}
\end{center}

% =============================================================================
% REFERENCES
% =============================================================================\renewcommand{\bibsection}{\section*{References}}

\end{document}